\documentclass[twoside,11pt]{article}
\usepackage{jmlr2e}
\usepackage{color}
\usepackage{amsmath,amssymb}
\usepackage{mathrsfs}
\usepackage{braket}
\usepackage{bbm}
\usepackage{comment}
\usepackage{graphicx}
\usepackage{caption}
\usepackage{subcaption}
\usepackage{natbib}
\usepackage{hyperref}

\newcommand{\RR}{\mathbb{R}}

\DeclareMathOperator{\EE}{\mathbf{E}}	
\DeclareMathOperator{\PP}{\mathbf{P}}	
\newcommand{\pp}{\ensuremath{\operatorname{p}}}

\newcommand{\Cp}[2]{\pp\left(\left. #1 \, \right| #2 \right)}

\newcommand{\CP}[2]{\PP\left(\left. #1 \, \right| #2 \right)}
\newcommand{\E}[1]{\EE\left[ #1  \right]}
\newcommand{\CE}[2]{\EE\left[\left. #1 \, \right| #2 \right]}

\usepackage{graphicx}

\newcommand{\BK}[1]{ {\left( #1 \right)} }
\newcommand{\sqBK}[1]{ {\left[ #1 \right]} }
\newcommand{\curBK}[1]{ {\left\{ #1 \right\}} }
\newcommand{\angleBK}[1]{ {\left < #1 \right >} }

\newcommand{\norm}[1]{\left\Vert #1 \right\Vert}
\newcommand{\abs}[1]{\left| #1 \right|}

\newcommand{\Normal}{\ensuremath{\operatorname{N}}}

\newcommand{\dist}{\overset{\mathcal{D}}{\sim}}

\newcommand{\err}{H}
\newcommand{\distance}{\mathrm{dist}}
\newcommand{\eqdef}{\overset{{\mbox{\tiny  def}}}{=}}

\newcommand{\A}{\mathcal{A}}
\newcommand{\X}{\mathcal{X}}
\newcommand{\indicator}{\mathbbm{1}}
\newcommand{\biasvariance}{\mathbb{B}}

\newcommand{\nData}{N}
\newcommand{\nSubData}{n}
\newcommand{\data}{x}
\newcommand{\noise}{\eta}
\renewcommand{\time}{m}

\renewcommand{\phi}{\varphi}
\renewcommand{\epsilon}{\varepsilon}

\newcommand{\scrB}{ \mathscr{B} }
\newcommand{\scrF}{ \mathscr{F} }
\newcommand{\scrR}{ \mathscr{R} }
\newcommand{\calF}{ \mathcal{F} }
\newcommand{\calO}{ \mathcal{O} }
\newcommand{\calC}{ \mathcal{C} }
\newcommand{\calU}{ \mathcal{U} }
\newcommand{\cali}{ \mathcal{I} }

\newcommand{\iid}{\overset{\text{i.i.d.}}{\sim}}
\newcommand{\ito}{\mathcal{I}}
\newcommand{\continuous}{\mathcal{C}}
\newcommand{\mesh}{\mathrm{mesh}}

\newtheorem{assumption}[theorem]{Assumption}
\newtheorem{rem}[theorem]{Remark}

\ShortHeadings{Consistency and fluctuations for stochastic gradient Langevin dynamics}{Teh, Thi{\'e}ry and Vollmer}
\firstpageno{1}

\begin{document}

\title{Consistency and fluctuations for stochastic gradient Langevin dynamics}
       
\author{\name Yee Whye Teh \email y.w.teh@stats.ox.ac.uk\\
       \addr Department of Statistics\\
       University of Oxford
       \AND
       \name Alexandre H. Thiery \email a.h.thiery@nus.edu.sg\\
       \addr Department of Statistics and Applied Probability\\
       National University of Singapore
       \AND
       \name Sebastian J. Vollmer \email vollmer@stats.ox.ac.uk\\
       \addr Department of Statistics\\
       University of Oxford
}

\editor{}

\maketitle
\begin{abstract}
Applying standard Markov chain Monte Carlo (MCMC) algorithms to large data sets is computationally expensive.
Both the calculation of the acceptance probability and the creation
of informed proposals usually require an iteration through the whole
data set. The recently proposed stochastic gradient Langevin dynamics (SGLD) method circumvents this problem by generating proposals which are only based on a subset of the data, by skipping the accept-reject step and by using decreasing step-sizes sequence $(\delta_m)_{m \geq 0}$. 

We provide in this article a rigorous mathematical framework for analysing this algorithm. We prove that, under verifiable assumptions, the algorithm is consistent, satisfies a central limit theorem (CLT) and its asymptotic bias-variance decomposition can be characterized by an explicit functional of the step-sizes sequence $(\delta_m)_{m \geq 0}$. We leverage this analysis to give practical recommendations for the notoriously difficult tuning of this algorithm: it is asymptotically optimal to use a step-size sequence of the type $\delta_m \asymp m^{-1/3}$, leading to an algorithm whose mean squared error (MSE) decreases at rate $\mathcal{O}(m^{-1/3})$.
\end{abstract}


\tableofcontents

%
%
\section{Introduction}
We are  entering the age of Big Data, where significant advances across
a range of scientific, engineering and societal pursuits hinge upon the
gain in understanding derived from the analyses of large scale data sets.
Examples include recent advances in genome-wide association studies
\citep{hirschhorn2005genome,mccarthy2008genome,wang2005genome}, speech
recognition \citep{hinton2012deep}, object recognition
\citep{krizhevsky2012imagenet}, and self-driving cars
\citep{thrun2010toward}.  As the quantity of data available has been
outpacing the computational resources available in recent years, there is
an increasing demand for new scalable learning methods, for example methods
based on stochastic optimization
\citep{robbins1951stochastic,srebro2010stochastic,sato2001online,hoffman2010online},
distributed computational architectures
\citep{ahmed2012scalable,neiswanger2013asymptotically,minsker2014robust},
greedy optimization 
\citep{harchaoui2014frank},
as well as the development of specialized computing systems supporting
large scale machine learning applications
\citep{gonzalez2014emerging}.

Recently, there has also been increasing interest in methods for Bayesian
inference scalable to Big Data settings.  Rather than attempting a single
point estimate of parameters typical in optimization-based or maximum
likelihood settings, Bayesian methods attempt to obtain characterizations
of the full posterior distribution over the unknown parameters and latent
variables in the model, hence providing better characterizations of the
uncertainties inherent in the learning process, as well as providing
protection against overfitting.  Scalable Bayesian methods proposed in the
recent literature include stochastic variational inference
\citep{sato2001online,hoffman2010online}, which applies stochastic
approximation techniques to optimizing a variational approximation to the
posterior, parallelized Monte Carlo
\citep{neiswanger2013asymptotically,minsker2014robust}, which distributes
the computations needed for Monte Carlo sampling across a large compute
cluster, as well as subsampling-based Monte Carlo
\citep{welling2011bayesian,ahn2012bayesian,korattikara2014austerity}, which
attempt to reduce the computational complexity of Markov chain Monte Carlo (MCMC)
methods by applying updates to small subsets of data.  

In this paper we study the asymptotic properties of the stochastic gradient
Langevin dynamics (SGLD) algorithm first proposed by
\citet{welling2011bayesian}.  SGLD is a subsampling-based MCMC algorithm
based on combining ideas from stochastic optimization, specifically using
small subsets of data to estimate gradients, with Langevin dynamics, a MCMC
method making use of gradient information to produce better parameter
updates.  \citet{welling2011bayesian} demonstrated that SGLD works well on a
variety of models and this has since been extended by
\citet{ahn2012bayesian,ahn2014distribuetd} and
\citet{patterson2013stochastic}.

The stochastic gradients in SGLD introduce approximations into the Markov
chain, whose effect has to be controlled by using a slowly decreasing
sequence of step sizes.  \citet{welling2011bayesian} provided an intuitive
argument that as the step-size decreases the variations introduced by the
stochastic gradients gets dominated by the natural stochasticity of
Langevin dynamics, the result being that the stochastic gradient
approximation should wash out asymptotically and that the Markov chain
should converge to the true posterior distribution.  

In this paper, we make this intuitive argument more precise by providing
conditions under which SGLD converges to the targeted posterior
distribution; we describe a number of characterizations of this
convergence.  Specifically, we show that estimators derived from SGLD are
consistent (Theorem \ref{thm.LLN}) and satisfy a central limit theorem (CLT) (Theorem \ref{thm.clt}); the bias-variance trade-off of the algorithm is discussed in details in Section \ref{sec.bias.variance}. In Section \ref{sec.diff.lim} we prove that, when observed on the right (inhomogeneous) time scale, the sample path of the
algorithm converges to a Langevin diffusion (Theorem \ref{thm.diff.lim}).   

Our analysis reveals that for a sequence of step-sizes with algebraic decay $\delta_m \asymp m^{-\alpha}$ the optimal choice, when measured in terms of rate of decay of the mean squared error (MSE), is given for $\alpha_\star = 1/3$; the choice $\delta_m \asymp m^{-\alpha_\star}$ leads to an algorithm that converges at rate $\mathcal{O}(m^{-1/3})$. This rate of convergence is worse than the standard Monte-Carlo $m^{-1/2}$-rate of convergence. This is  not due to the stochastic gradients used in SGLD, but rather to the  decreasing step-sizes.


These results are asymptotic in the sense that they characterise the behaviour of the algorithm as the number of steps approaches infinity. Therefore they do not necessarily translate into any insight into the behaviour for finite computational budgets which is the regime in which the SGLD might provide computational gains over alternatives. The mathematical framework described in this article show that the SGLD is a sound algorithm, an important result that has been missing in the literature. \\

In the remainder of this article, the notation $\Normal(\mu, \sigma^2)$ denotes a Gaussian distribution with mean $\mu$ and variance $\sigma^2$.
For two positive functions $f,g: \RR \to [0,\infty)$, one writes $f \lesssim g$ to indicate that there exists a positive constant $C > 0$ such that $f(\theta) \leq C \, g(\theta)$; we write $f \asymp g$ if $f \lesssim g \lesssim f$.
For a probability measure $\pi$ on a measured space $\X$, a measurable function $\phi: \X \to \RR$ and a measurable set $A \subset \X$, we define $\pi(\phi;A)=\int_{\theta \in A} \phi(\theta) \, \pi(d \theta)$ and $\pi(\phi) = \pi(\phi; \X)$. Finally, densities of probability distributions on $\RR^d$ are implicitly assumed to be defined with respect to the usual $d$-dimensional Lebesgue measure.

\subsection*{Acknowledgement}
SJV and YWT  acknowledge EPSRC for research funding through grant  EP/K009850/1 and EP/K009362/1. AHT is grateful for financial support in carrying out this research from a Singaporean MoE grant.

%
%
\section{Stochastic Gradient Langevin Dynamics}
\label{sec:sgld}
Many MCMC algorithms evolving in a continuous state space, say $\RR^d$, can be realised as discretizations of a continuous time Markov process $(\theta_t)_{t\ge 0}$.  An example of such a continuous time process, which is central to SGLD as well as many other algorithms, is the Langevin diffusion, which is given by the stochastic differential equation
\begin{equation}
	\label{eq:overdampedLangevin}
	d \theta_t = \frac12 \, \nabla \log \pi(\theta_t) \, dt + dW_t,
\end{equation}
where $\pi:\RR^d \to (0,\infty)$ is a probability density and $(W_t)_{t\ge 0}$ is a standard Brownian motion in $\RR^d$.  
The linear operator $\A$ denotes the generator of the Langevin diffusion \eqref{eq:overdampedLangevin}: for a twice continuously differentiable test function $\phi:\RR^d \to \RR$, 
\begin{equation} \label{eq.generator}
\A \phi(\theta) = \frac{1}{2} \angleBK{\nabla \log \pi(\theta), \nabla \phi(\theta)} + \frac{1}{2} \Delta \phi(\theta),
\end{equation}
where $\Delta \phi \eqdef \sum_{i=1}^d \nabla^2_{i} \phi$ denotes the standard Laplacian operator.
The motivation behind the choice of Langevin diffusions is that, under certain conditions, they are ergodic with respect to the distribution $\pi$; for example, \citep{roberts1996exponential,stramer1999langevin1, stramer1999langevin2,mattingly2002ergodicity} describe drift conditions of the type described in Section \ref{sec.stability} that ensure that the total variation distance from stationarity of the law at time $t$ of the Langevin diffusion \eqref{eq:overdampedLangevin} decreases to zero exponentially quickly as $t \to \infty$.


%

Given a time-step $\delta > 0$ and a current position $\theta_t$, it is often straightforward to simulate a random variable $\theta_\star$ that is approximately distributed as the law of $\theta_{t + \delta}$ given $\theta_t$. For stochastic differential equations, the Euler-Maruyama scheme \citep{maruyama1955continuous} might be the simplest approach for approximating the law of $\theta_{t + \delta}$.  For a Langevin diffusion this reads
\begin{align} \label{eq:euler.maruyama}
\theta_\star &= \theta_t + \frac{1}{2}\delta\, \nabla \log \pi(\theta_t)   + \delta^{1/2} \, \eta
\end{align}
for a standard $d$-dimensional centred Gaussian random variable $\eta$.
To fully correct the discretization error, one can adopt a Metropolis-Hastings  accept-reject mechanism. The resulting algorithm is usually referred to as the Metropolis-Adjusted-Langevin algorithm (MALA) \citep{roberts1996exponential}.  Other discretizations can be used as proposals.  For example, the random walk Metropolis-Hastings algorithm uses the discretization of a standard Brownian motion as the proposal, while the  Hamiltonian Monte Carlo (HMC) algorithm \citep{duane1987hybrid} is based on discretizations of an Hamiltonian system of differential equations.  See the excellent review of \citet{neal2010mcmc} for further information.

In this paper, we shall consider the situation where the target $\pi$ is the density of the posterior distribution under a Bayesian model where there are $N\gg 1$ i.i.d.\ observations, the so called Big Data regime,
\begin{equation} \label{eq.posterior.iid}
\pi(\theta) \propto \pp_0(\theta) \, \prod_{i=1}^N \Cp{y_i}{ \theta }.
\end{equation}
Here, both computing the gradient term $\nabla \log \pi(\theta_t)$ and evaluating the Metropolis-Hastings acceptance ratio require a computational budget that scales unfeasibly as $\calO(N)$.
One approach is to use a standard random walk proposal instead of Langevin dynamics, and to efficiently approximating the Metropolis-Hastings accept-reject mechanism using only a subset of the data \citep{korattikara2014austerity,BaDoHo14}.  

This paper is concerned with stochastic gradient Langevin dynamics (SGLD), an alternative approach proposed by \citet{welling2011bayesian}.  This follows the opposite route and chooses to completely avoid the computation of the Metropolis-Hastings ratio. By choosing a discretization of the Langevin diffusion \eqref{eq:overdampedLangevin} with a sufficiently small step-size $\delta \ll 1$, because the Langevin diffusion is ergodic with respect to $\pi$, the hope is that even if the Metropolis-Hastings accept-reject mechanism is completely avoided, the resulting Markov chain still has an invariant distribution that is close to $\pi$.   Choosing a decreasing sequence of step-sizes $\delta_m \to 0$ should even allow us to converge to the exact posterior distribution. To further make this approach viable in large $N$ settings, the gradient term $\nabla \log \pi(\theta)$ can be further approximated using a subsampling strategy. For an integer $1 \leq n \leq N$ and a random subset $\tau \eqdef (\tau_1, \ldots, \tau_n)$  of $[N] \equiv\{1, \ldots, N\}$ generated by sampling with or without replacement from $[N]$, the quantity
\begin{equation} \label{eq.unbiased.iid}
	\nabla \log \pp_0(\theta) + \frac{N}{n} \, \sum_{i=1}^n \nabla \log \pp(x_{\tau_i} \mid \theta)
\end{equation}
is an unbiased estimator of $\nabla \log \pi(\theta)$. Most importantly, this stochastic estimate can be computed with a computational budget that scales as $\calO(n)$ with $n$ potentially much smaller than $N$.  Indeed, the larger the quotient $n/N$, the smaller the variance of this estimate.  

Stochastic gradient methods have a long history in optimisation and machine learning and are especially relevant in the large dataset regime considered in this article  \citep{robbins1951stochasticApprox,bottou2010large,Hoffman2013SVI}.
In this paper we will adopt a slightly more general framework and assume that one can compute an unbiased estimate $\widehat{\nabla \log \pi}(\theta,\calU)$ to the gradient $\nabla \log \pi(\theta)$, where $\calU$ is an auxiliary random variable which contains all the randomness involved in constructing the estimate.  Without loss of generality we may assume (although this is unnecessary) that $\calU$ is uniform on $(0,1)$. The unbiasedness of the estimator $\widehat{\nabla \log \pi}(\theta,\calU)$ means that
\begin{equation} \label{eq.unbiased.estimate}
	\E{ \err(\theta,\calU) }=0
	\qquad \textrm{with} \qquad
	\err(\theta,\calU) \; \eqdef\ \; \widehat{\nabla \log \pi}(\theta,\calU) - \nabla \log \pi(\theta).
\end{equation}
In summary, the SGLD algorithm can be described as follows. For a sequence of asymptotically vanishing time-steps $(\delta_m)_{m \geq 0}$ and an initial parameter $\theta_0 \in \RR^d$, if the current position is $\theta_{m-1}$, the next position $\theta_m$ is defined though the recursion
\begin{equation}\label{eq.sgld}
\theta_m = \theta_{m-1} + \frac{1}{2}\delta_m\, \widehat{\nabla \log \pi}(\theta_{m-1}, \calU_m)   + \delta_m^{1/2} \, \eta_m
\end{equation}
for an i.i.d.\ sequence $\eta_{m}\sim\Normal(0,I_d)$, and an independent and i.i.d.\ sequence $\calU_{m}$ of auxiliary random variables.  This is the equivalent of the Euler-Maruyama discretization  \eqref{eq:euler.maruyama} of the Langevin diffusion \eqref{eq:overdampedLangevin} with a decreasing sequence of step-sizes and a stochastic estimate to the gradient term. The analysis presented in this article assumes for simplicity that the initial position $\theta_0$ of the algorithm is deterministic; in the simulation study of Section \ref{sec.numerics}, the algorithms are started at the MAP estimator. Indeed, more general situations could be analysed with similar arguments at the cost of slightly less transparent proofs. Note that the process $(\theta_m)_{m \geq 0}$ is a non-homogeneous Markov chain, and many standard analysis techniques for homogeneous Markov chains do not apply.

For a test function $\phi: \RR^d \to \RR$, the expectation of $\phi$ with respect to the posterior distribution $\pi$ can be approximated by the weighted sum
\begin{align}\label{eq:empMeasure}
	\pi_{m}(\phi)
	&\eqdef
	\frac{  \delta_1 \, \phi(\theta_0) + \ldots + \delta_m \, \phi(\theta_{m-1}) \, }{ T_m }
\end{align}
with $T_m = \delta_1 + \ldots + \delta_m$. The quantity $\pi_{m}(\phi)$ thus  approximates the ergodic average $T_m^{-1} \int_{0}^{T_m} \phi(\theta_t) \, dt$ between time zero and $t=T_m$.
%
During the course of the proof of our fluctuation Theorem \ref{thm.clt}, we will need to consider more general averaging schemes than the one above. Instead, for a general positive sequence of weights $\omega=(\omega_m)_{m \geq 1}$, we define the $\omega$-weighted sum
\begin{align}\label{eq:empMeasure:weighted}
\pi^{\omega}_{m}(\phi)
&\eqdef
\frac{ \omega_1 \, \phi(\theta_0) + \ldots + \omega_m \, \phi(\theta_{m-1})}{\Omega_m}
\end{align}
with $\Omega_m \eqdef \omega_1 + \ldots + \omega_m$.
Indeed, $\pi_m^{\omega}(\phi) = \pi_m(\phi)$ in the particular case $(\omega_m)_{m \geq 1} = (\delta_m)_{m \geq 1}$; we will consider the weight sequence $\omega = \{\delta_m^2\}_{m \geq 1}$ in the proof of Theorem \ref{thm.clt}.

Let us mention several directions that can be explored to improve upon the basic SGLD algorithm explored in this paper. Langevin diffusions of the type $d \theta_t = \textrm{drift}(\theta_t) \, dt + M(\theta_t) \, dW_t$, reversible with respect to the posterior distribution $\pi$, can be constructed for various choices of positive definite volatility matrix function $M: \RR^d \to \RR^{d,d}$. Note nonetheless that, for a non-constant volatility matrix function $\theta \mapsto M(\theta)$, the drift term typically involves derivatives of $M$.
Concepts of information geometry \citep{amari2007methods} give principled ways \citep{livingstone2014information} of choosing the volatility matrix function $M$; when the Fisher information matrix is used, this leads to the  Riemannian manifold MALA algorithm \citep{MR2814492}. This approach has recently been applied to the Latent Dirichlet Allocation model for topic modelling \citep{PatTeh2013a}. For high-dimensional state spaces $d \gg 1$, one can use a constant volatility function $M$, also known in this case as the preconditioning matrix, for taking into account the information contained in the prior distribution $\pp_0$ in the hope of obtaining better mixing properties \citep{beskos2008mcmc,cotter2013mcmc}; infinite dimensional limits are obtained in \citep{pillai2012optimal,hairer2011spectral}.
Under an uniform-ellipticity condition and a growth assumption on the volatility matrix function $M:\RR^d \to \RR^{d,d}$, we believe that our framework could, at the cost of increasing complexity in the proofs, be extended to this setting. To avoid the slow random walk behaviour of Markov chains based on discretization of reversible diffusion processes, one can use instead discretizations of an Hamiltonian system of ordinary differential equations \citep{duane1987hybrid,neal2010mcmc}; when coupled with the stochastic estimates to the gradient above described, this leads to the stochastic gradient Hamiltonian Monte Carlo algorithm of \citep{chen2014stochastic}.

In the rest of this paper, we will build a rigorous framework for understanding the properties of this SGLD algorithm, demonstrating that the heuristics and numerical evidences presented in \citet{welling2011bayesian} were indeed correct. 

\section{Assumptions and Stability Analysis}

This section starts with the basics assumptions we will need for the asymptotic results to follow, and illustrates some of the potential stability issues that may occur, would the SGLD algorithm be applied without care. 

\subsection{Basic Assumptions}
%
%
Throughout this text, we assume that the sequence of step-sizes $\delta = (\delta_m )_{m \geq 1}$ satisfies the following usual assumption.
\begin{assumption}
\label{ass:step-sizes}
The step-sizes $\delta = (\delta_m )_{m \geq 1}$ form a decreasing sequence with
\begin{equation*}
\lim_{m \to \infty} \, \delta_m = 0
\qquad \textrm{and} \qquad
\lim_{m \to \infty} \, T_m = \infty.
\end{equation*}
\end{assumption}
Indeed, this assumption is easily seen to also be necessary for the Law of Large Numbers of Section \ref{sec.LLN} to hold.
%
Furthermore, we will need at several occasions to assume the following assumption on the oscillations of a sequence of step-sizes $( \omega_m )_{m \geq 1}$.
%
%
%
\begin{assumption}
\label{ass:step-sizes:weighted}
The step-sizes sequence $( \omega_m )_{m \geq 1}$ is such that $\omega_m \to 0$ and $\Omega_m \to \infty$ and
\begin{equation*}
\lim_{m \to \infty} \, \sum_{m \geq 1} \big| \Delta(\omega_m / \delta_m) \big| \, / \, \Omega_m < \infty
\qquad \textrm{and}\qquad
\sum_{m \geq 1} \omega^2_m / [\delta_m \Omega^2_m] < \infty.
\end{equation*}
where $\Delta(\omega_m / \delta_m) \eqdef \omega_{m+1} / \delta_{m+1} - \omega_m / \delta_m$.
\end{assumption}
\begin{rem} \label{rem.weights.power} 
Assumption \ref{ass:step-sizes:weighted} holds if 
$\delta = (\delta_m)_{m \geq 1}$ satisfies Assumption \eqref{ass:step-sizes} and the weights are defined as $\omega_m =\delta^p_m$, for some some exponent $p \geq 1$ small enough for $\Omega_m \to \infty$. This is because the first sum is less than $\sum_{m \geq 1} \big| \Delta(\omega_m / \delta_m) \big|/\Omega_1 = \delta_1^{p-1} / \Omega_1$, while the finiteness of the second sum can be seen as follows:
\begin{eqnarray*}
\sum_{m \geq 1} \omega^2_m / \BK{ \delta_m \Omega^2_m }
&\lesssim&
1+\sum_{m \geq 2} \BK{ \omega_m / \delta_m }^2 \, \BK{ 1/\Omega_{m-1} - 1/\Omega_m }\\
&\lesssim& 
1+\sum_{m \geq 2} \BK{ 1/\Omega_{m-1} - 1/\Omega_m } = 1 + 1/\Omega_1.
\end{eqnarray*}
For any exponents $0 < \alpha < 1$ and $0 < p < 1/\alpha$ the sequences $\delta_m = (m_0+m)^{-\alpha}$ and $\omega_m = \delta_m^p$ satisfy both Assumption \ref{ass:step-sizes} and Assumption \ref{ass:step-sizes:weighted}.
\end{rem}

%
%
\subsection{Stability}
\label{sec.stability}

Under assumptions on the tails of the posterior density $\pi$, the Langevin diffusion \eqref{eq:overdampedLangevin} is non-explosive and for any starting position $\theta_0 \in \RR^d$ the total-variation distance $d_{\textrm{TV}}\big(\PP(\theta_t \in \cdot), \pi\big)$ converges to zero as $t \to \infty$. For instance, Theorem $2.1$ of \citep{roberts1996exponential} shows that it is sufficient to assume that the drift term  satisfies the condition $(1/2) \, \angleBK{\nabla \log \pi(\theta), \theta} \leq \alpha \|\theta\|^2 + \beta$ for some constants $\alpha, \beta> 0$. We refer the interested reader to \citep{roberts1996exponential,stramer1999langevin1,stramer1999langevin2,roberts2002langevin,mattingly2002ergodicity}
for a detailed study of the convergence properties of the Langevin diffusion \eqref{eq:overdampedLangevin}. 

Unfortunately, stability of the continuous time Langevin diffusion does not always translate into good behaviour for its Euler-Maruyama discretization. For example, even if the drift term points towards the right direction in the sense that $\angleBK{\nabla \log \pi(\theta), \theta} < 0$ for every parameter $\theta$, it might happen that the magnitude of the drift term is too large so that the Euler-Maruyama discretization \emph{overshoots} and becomes unstable.  In a one dimensional setting, this would lead to a Markov chain that diverges in the sense that the sequence $(\theta_m)_{m \geq 0}$ alternates between taking arbitrarily large positive and negative values. Lemma $6.3$ of \citep{mattingly2002ergodicity} gives such an example with a target density $\pi(\theta) \propto \exp\{-\theta^4\}$.  See also Theorem $3.2$ of \citep{roberts1996exponential} for examples of the same flavours.

Guaranteeing stability of the Euler-Maruyama discretization requires stronger Lypanunov type conditions. At a heuristic level, one must ensure that the drift term $\nabla \log \pi(\theta)$ points towards the centre of the state space.  In addition,  the previous discussion indicates that one must also ensure that the magnitude of this drift term is not too large. The following assumptions satisfy both heuristics, and we will show are enough to guarantee that the SGLD algorithm is consistent, with asymptotically Gaussian fluctuations.
%
%
\begin{assumption} \label{ass:Lyap} 
The drift term $\theta \mapsto \frac{1}{2} \, \nabla \log \pi(\theta)$ is continuous. There exists a  Lyapunov function $V: \RR^d \to [1,\infty)$ that tends to infinity as $\|\theta\| \to \infty$, is twice differentiable with bounded second derivatives, and satisfies the following conditions.
\begin{enumerate}
\item 
There exists an exponent $p_H \geq 2$ such that
\begin{equation} \label{eq.bound.H}
\E{ \norm{ H(\theta, \calU) }^{2p_H}} \lesssim V^{p_H}(\theta).
\end{equation}
This implies that $\E{ \|H(\theta, \calU) \|^{2p} } \lesssim V^{p}(\theta)$ for any exponent $0 \leq p \leq p_H$.
\item For every $\theta \in \RR^d$ we have
\begin{equation} \label{eq.lyapunov.size}
\norm{ \nabla V(\theta) }^2 + \norm{ \nabla \log \pi(\theta) }^2
\; \lesssim \;   V(\theta).
\end{equation}
\item
There are constants $\alpha, \beta > 0$ such that for every $\theta \in \RR^d$ we have
\begin{equation} \label{eq.lyapunov.drift}
\frac12 \, \angleBK{ \nabla V(\theta), \nabla \log \pi(\theta)}  \;\leq \; -\alpha \, V(\theta)+\beta.
\end{equation}
\end{enumerate}
\end{assumption}

Equation \eqref{eq.lyapunov.drift} ensures that on average the drift term $\widehat{\nabla \log \pi}(\theta)$ points towards the centre of the state space, while equations \eqref{eq.bound.H} and \eqref{eq.lyapunov.size} provide control on the magnitude of the (stochastic) drift term. The drift condition \eqref{eq.lyapunov.drift} implies in particular that the Langevin diffusion \eqref{eq:overdampedLangevin} converges exponentially quickly towards the equilibrium distribution $\pi$  \citep{mattingly2002ergodicity,roberts1996exponential}. 
The proof of the Law of Large Numbers (LLN) and the Central Limit Theorem (CLT) both exploit the following Lemma.
%
%
\begin{lemma} [Stability] \label{lem:stability}
Let the step-sizes $(\delta_m)_{m \geq 1}$ satisfy Assumption \ref{ass:step-sizes} and suppose that the stability Assumptions \ref{ass:Lyap} hold. For any exponent $0\le p \leq p_H$ the following bounds hold almost surely,
\begin{equation} \label{eq.stability.estimate}
\sup_{m \geq 1} \; \pi_m(V^{p/2}) \; < \; \infty
\qquad \textrm{and} \qquad
\sup_{m \geq 1} \; \E{ V^p(\theta_m) } \; < \; \infty.
\end{equation}
Moreover, for any exponent $0\le p \leq p_H$ we have $\pi(V^p) < \infty$.
If the sequence of weights $(\omega_m)_{m \geq 1}$ satisfies Assumption \ref{ass:step-sizes:weighted} the following holds almost surely,
\begin{equation} \label{eq.stability.estimate.weighted}
\sup_{m \geq 1} \; \pi^{\omega}_m(V^{p/2}) \; < \; \infty
\end{equation}
\end{lemma}
The technical proof can be found in Section \ref{sec.proof.lem.stability}. The idea is to leverage condition \eqref{eq.lyapunov.drift} in order to establish that the function $V^p$ satisfies both discrete and continuous drift conditions.

\subsection{Scope of the analysis}
For a posterior density $\pi$ of the form \eqref{eq.posterior.iid} and the usual unbiased estimate to $\nabla \log \pi$ described in Equation \eqref{eq.unbiased.iid}, to establish that Equations \eqref{eq.bound.H} and \eqref{eq.lyapunov.size} hold it suffices to verify that the prior density $p_0$ is such that $\norm{ \nabla \log \pp_0(\theta) }^2 \lesssim V(\theta)$ and that for any index $1 \leq i \leq N$ the likelihood term $\Cp{y_i}{\theta}$ is such that
\begin{equation*}
\norm{ \nabla \log \Cp{y_i}{\theta} }^{2 \, p_H} \lesssim V^{p_H}(\theta).
\end{equation*}
Indeed, in these circumstances, we have 
$\norm{ H(\theta, \calU) }^{2 p_H} \lesssim \sum_{i=1}^N \, \norm{ \nabla \log \pp(y_i \mid \theta) }^{2 p_H}$. Several such examples are described in Section \ref{sec.numerics}. 

It is important to note that the drift Condition \eqref{eq.lyapunov.drift} typically does not hold for distributions with heavy tails such that $\nabla \log \pi(x) \to 0$ as $\norm{x} \to \infty$ \citep{roberts1996exponential}. For example, the standard MALA algorithm is not geometrically ergodic when $\nabla \log \pi(x)$ converges to zero as $\norm{x} \to \infty$ (Theorem $4.3$ of \citep{roberts1996exponential}); indeed, the analysis of standard local-move MCMC algorithms when applied to target densities with heavy tails is delicate and typically necessitate other tools \cite{stramer1999langevin2,jarner2007convergence,kamatani2014rate} than the approach based on drift conditions of the type \eqref{eq.lyapunov.drift}. The analysis of the properties of the SGLD algorithm when applied to such heavy tail densities is out of the scope of this article. It is important to note that many more complex scenarios involving high-dimensionality, multi-modality, non-parametric settings where the complexity of the target distribution increases with the size of the data, or combination thereof, are examples of interesting and relevant situations where our analysis typically does not apply; analysing the SGLD algorithm when applied to these challenging target distributions is well out of the scope of this article.

%
%
\section{Consistency} \label{sec.LLN}

The problem of estimating the invariant distribution of a stochastic differential equation by using a diminishing step-size Euler discretization has been well explored in the literature \citep{lamberton2002recursive,lamberton2003recursive,
lemaire2007adaptive,panloup2008recursive,
pages2012ergodic}, while \citep{mattingly2002ergodicity} studied the bias and variance of similar algorithms when fixed step-sizes are used instead.  We leverage some of these techniques and adapt it to our setting where the drift term can only be unbiasedly estimated, and establish in this section that the SGLD algorithm is consistent under Assumptions \ref{ass:step-sizes} and \ref{ass:Lyap}. More precisely, we prove that almost surely the sequence $(\pi_m)_{m \geq 1}$ defined in Equation \eqref{eq:empMeasure} converges weakly towards $\pi$.  Specifically, under growth assumptions on a test function $\phi: \RR^d \to \RR$, the following strong law of large numbers holds almost surely,
\begin{equation*}
\lim_{m \to \infty} \;
 \frac{ \delta_1 \, \phi(\theta_0) + \ldots + \delta_m \, \phi(\theta_m)}{T_m}  \; = \; \int_{\RR^d} \, \phi(\theta) \, \pi(d \theta),
\end{equation*}
with a similar result for $\omega$-weighted empirical averages, under assumptions on the weight sequence $\omega$. The proofs of several results of this paper make use of the following elementary lemma.
\begin{lemma} \label{lem.MR}
Let $\BK{\Delta M_k}_{k \geq 0}$ and $\BK{R_k}_{k \geq 0}$ be two sequences of random variables adapted to a filtration $\BK{\calF_{k}}_{k \geq 0}$ and let $\BK{\Gamma_k}_{k \geq 0}$ be an increasing sequence of positive real numbers. The limit
\begin{equation} \label{eq.martingale.type}
	\lim_{m \to \infty} \frac{\sum_{k=0}^m \Delta M_k + R_k}{T_m} 
	\;=\; 0
\end{equation}
holds almost surely if the following two conditions are satisfied.
\begin{enumerate}
\item The process $M_m = \sum_{k \leq m} \Delta M_k$ is a martingale, i.e. $\CE{\Delta M_k}{\calF_k}=0$ and
\begin{equation} \label{eq.martingale.type.M.part}
\lim_{k \to \infty} \sum_{k \geq 0} \frac{\EE \sqBK{\abs{\Delta M_k}}^2 }{T_k^2} \; < \; \infty.
\end{equation}
\item The sequence $\BK{R_k}_{k \geq 0}$ is such that
\begin{equation} \label{eq.martingale.type.R.part}
\lim_{k \to \infty} \sum_{k \geq 0} \frac{\EE \sqBK{\abs{R_k}} }{T_k} \; < \; \infty.
\end{equation}
\end{enumerate}
\end{lemma}
The above lemma, whose proof can be found in the appendix \ref{sec.proof.lem.MR}, is standard; \cite{lamberton2002recursive} also follows this route to prove several of their results.

%
%
\begin{theorem} \label{thm.LLN} {\bf (Consistency)}
Let the step-sizes satisfy Assumption \eqref{ass:step-sizes} and suppose that the stability Assumptions \ref{ass:Lyap} hold for a Lyapunov function $V:\RR^d \to [1,\infty)$. Let $0 \leq p < p_H/2$ and $\phi:\RR^d \to \RR$ be a test function such that $| \phi(\theta) | / V^{p}(\theta)$ is globally bounded. Then the following limit holds almost surely:
\begin{equation} \label{eq.ergodic.unbounded}
\lim_{m \to \infty} \; \pi_m(\phi) \;=\; \pi(\phi).
\end{equation}
If in addition the sequence of weights $\{\omega_m\}_{m \geq 1}$ satisfies Assumption \eqref{ass:step-sizes:weighted}, a similar result holds almost surely for the $\omega$-weighted ergodic average:
\begin{equation} \label{eq.weighted.LLN}
\lim_{m \to \infty} \; \pi^{\omega}_m(\phi) \;=\; \pi(\phi).
\end{equation}
\end{theorem}

%
%
\begin{proof}
In the following, we write $\EE_k \sqBK{\, \cdot \,}$ and $\PP_k \BK{ \, \cdot \,}$ to denote the conditional expectation $\CE{\, \cdot \, }{ \theta_k}$ and conditional probability $\CP{\, \cdot \, }{ \theta_k}$ respectively. We use the notation $\Delta \theta_k \eqdef (\theta_{k+1}-\theta_k)$. Finally, for notational convenience, we only present the proof in the scalar case $d=1$, the multidimensional case being entirely similar. 
We will give a detailed proof of Equation \eqref{eq.ergodic.unbounded} and then briefly describe how the more general Equation \eqref{eq.weighted.LLN} can be proven using similar arguments.
To prove Equation \eqref{eq.ergodic.unbounded}, we first show that the sequence $(\pi_m)_{m \geq 1}$ almost surely converges weakly to $\pi$. Equation \eqref{eq.ergodic.unbounded} is then proved in a second stage.\\

\noindent
{\bf Weak convergence of $(\pi_m)_{m \geq 1}$}.
To prove that  almost surely the sequence $(\pi_m)_{m \geq 1}$ converges weakly towards $\pi$ it suffices to prove that the sequence is almost surely weakly pre-compact and that any weakly convergent subsequence of $(\pi_m)_{m \geq 0}$ necessarily (weakly) converges towards $\pi$. By Prokhorov's Theorem \citep{Billingsley95} and Equation \eqref{eq.stability.estimate}, because the Lyapunov function $V$ goes to infinity as $\|\theta\| \to \infty$, the sequence $(\pi_m)_{m \geq 1}$ is almost surely weakly pre-compact. It thus remains to show that if a subsequence converges weakly to a probability measure $\pi_{\infty}$ then $\pi_{\infty}=\pi$.

Since the Langevin diffusion \eqref{eq:overdampedLangevin} has a unique strong solution and its generator $\A$ is uniformly elliptic, Theorem $9.17$ of Chapter $4$ of \citep{Ethier1986Markov} yields that it suffices to verify that 
for any smooth and compactly supported test function $\phi: \RR \to \RR$ and any limiting distribution $\pi_{\infty}$ of the sequence $(\pi_m)_{m \geq 1}$ the following holds,
\begin{equation} \label{eq.invariance.by.generator}
\pi_{\infty}(\A \phi) = 0.
\end{equation}
To prove Equation \eqref{eq.invariance.by.generator} we use the following  decomposition of $\pi_m (\A \phi)$,
\begin{equation} \label{eq.pi.m.decomposition}
\curBK{ \frac{\sum_{k=1}^m \EE_{k-1}[\phi(\theta_k)-\phi(\theta_{k-1})]}{T_m}}
\;-\;
\curBK{  \frac{\sum_{k=1}^m \EE_{k-1}[\phi(\theta_k)-\phi(\theta_{k-1})]}{T_m} - \pi_m(\A \phi)}.
\end{equation}
%
%
%
\begin{itemize}
\item
Let us prove that the first term of \eqref{eq.pi.m.decomposition} converges almost surely to zero. The numerator is equal to the sum of $\sum_{k=1}^m \EE_{k-1}[\phi(\theta_k)] - \phi(\theta_k)$ and $\phi(\theta_m) - \phi(\theta_0)$. By boundedness of $\phi$, the term $\curBK{ \phi(\theta_m) - \phi(\theta_0) } / T_m$ converges almost surely to zero.
By Lemma \ref{lem.MR}, to conclude is suffices to show that the martingale difference terms $\EE_{k-1}\sqBK{ \phi(\theta_k)} - \phi(\theta_k)$ are such that
\begin{equation*}
	\sum_{k \geq 1} \frac{  \EE \sqBK{ \abs{ \EE_{k-1} \sqBK{ \phi(\theta_k)} - \phi(\theta_k) }^2 }}{ T^2_k} \; < \; \infty.
\end{equation*} 
Because $\phi$ is Lipschitz, it suffices to prove that$\sum_{k \geq 1} \EE\BK{ \norm{ \theta_{k+1} - \theta_k }^2 } / T_k^2$ is finite.
The stability Assumption \ref{ass:Lyap} and Lemma \ref{lem:stability} imply that the supremum $\sup_m \, \EE \sqBK{ V(\theta_m) }$ is finite.
Since $\EE_k \sqBK{ \norm{ \theta_{k+1} - \theta_k }^2 } \lesssim \delta^2_{k+1} \, V(\theta) + \delta_{k+1}$, it follows that $\EE \BK{ \norm{ \theta_{k+1} - \theta_k }^2 }$ is less than a constant multiple of $\delta_{k+1}$. Under Assumption \ref{ass:step-sizes}, because the telescoping sum $\sum_{k \geq 1} T^{-1}(k)-T^{-1}(k+1)$ is finite, the sum $\sum_{k \geq 1} \delta_k / T_k^2$ is finite.  This concludes the proof that the first term in \eqref{eq.pi.m.decomposition} converges almost surely to zero.

\item 
The second term of \eqref{eq.pi.m.decomposition} equals $\big(R_0 + \ldots +  R_{m-1} \big) / T_m$ with
\begin{equation} \label{eq.R}
R_k \eqdef \EE_k \sqBK{ \phi(\theta_{k+1}) - \phi(\theta_k) } - \A \phi(\theta_k) \, \delta_{k+1}.
\end{equation}
We now show that there exists a constant $C$ such that the bound $|R_k| \leq C \, \delta_{k+1}^{3/2}$ holds for any $k \geq 0$. To do so, let $K>0$ be such that the support of the test function $\phi$ is included in the compact set $\Omega = [-K,K]$. We examine two cases separately.
\begin{itemize}
\item
If $|\theta_k| > K+1$ then $\phi(\theta_k) = \A \phi(\theta_k)=0$ so that $|R_k| \leq \|\phi\|_{\infty} \times \PP_k(\theta_{k+1} \in \Omega)$. Since $\theta_{k+1}-\theta_k = \curBK{ \frac{1}{2} \nabla \log \pi(\theta_k) + \err(\theta_k, \calU) } \, \delta_{k+1} + \sqrt{\delta_{k+1}} \, \eta$ we have
\begin{align*}
\PP_k(\theta_{k+1} \in \Omega) 
&\leq
\mathbb{I}\BK{ \left|\frac{1}{2} \nabla \log \pi(\theta_k) \right| \geq \frac{\distance(\theta_k, \Omega)}{3 \, \delta_{k+1}} }\\
&\qquad +
\PP_k\BK{|\err(\theta_k,\calU)| \geq \frac{\distance(\theta_k, \Omega)}{3 \, \delta_{k+1}} }
+\PP_k\BK{|\eta| \geq \frac{\distance(\theta_k, \Omega)}{3 \, \sqrt{\delta_{k+1}}} }.
\end{align*}
We have used the notation $\mathbb{I}(A)$ for denoting the indicator function of the event $A$.
Under Assumption \ref{ass:Lyap} we have $|\nabla \log \pi(\theta)| \lesssim V(\theta)^{1/2} \lesssim 1+\|\theta\|$ so that the quotient $|\nabla \log \pi(\theta)| / \distance(\theta, \Omega)$ is bounded on the set $\curBK{ \theta : |\theta|>K }$; this shows that the first term equals zero for $\delta_k$ small enough. To prove that the second term is bounded by a constant multiple of $\delta_{k+1}^2$, it suffices to use Markov's inequality and the fact that $\EE[\err(\theta_k,\calU)^2] / \distance^2(\theta, \Omega)$ is bounded on $\{\theta : |\theta|>K\}$; this is because $\EE[\err(\theta_k,\calU)^2]$ is less than a constant multiple of $V(\theta)$ and $V(\theta) \lesssim 1 + \|\theta\|^2$ by Assumption \ref{ass:Lyap}. The third term is less than a constant multiple of $\delta_{k+1}^2$ by Markov's inequality and the fact that $\eta$ has a finite moment of order four.

\item
If $|\theta_k| \leq K+1$, we decompose $R_k$ into two terms.
A second order Taylor formula yields 
\begin{align*}
R_k 
&= \frac12 \, \delta^2_{k+1} \, \phi^{''}(\theta_k) \, \curBK{ [\nabla \log \pi(\theta_k)]^2 + \EE_k \sqBK{ \err^2(\theta_k,\calU) } } \\
&\qquad + (1/2) \, \EE_k\sqBK{ (\Delta \theta_k)^3 \, \int_{0}^1 \phi^{'''}(\theta_k + u \, \Delta \theta_k) \, (1-u)^2 \, du }\\
&= R_{k,1} + R_{k,2}.
\end{align*}
Under Assumption \ref{ass:Lyap}, the quantities $[\nabla \log \pi(\theta_k)]^2$ and $\EE[\err^2(\theta_k,\calU)]$ are upper bounded by a constant multiple of $V(\theta_k)$. Since the function $\theta \mapsto \phi^{''}(\theta) \, V(\theta)$ is globally bounded (because continuous with compact support) this shows that $R_{k,1}$ is less than a constant multiple of $\delta_{k+1}^2$. Since $|\theta_k| \leq K+1$, the bounds $\EE[\err^3(\theta,\calU)] \lesssim V^{3/2}(\theta)$ and $\sup_{k \geq 0} \EE[V^{3/2}(\theta_k)] < \infty$ (see Lemma \ref{lem:stability}) yield that $\EE_k |\Delta \theta_k|^3 \leq 9 \,\overline{C} \, (\delta_{k+1}^3+\delta_{k+1}^{3/2}) \lesssim \delta_{k+1}^{3/2}$ with 
\begin{equation*}
\overline{C} = 1+\sup_{\theta : |\theta|<K+1} \abs{ \nabla \log \pi(\theta) }^3 + \EE \sqBK{ \abs{ \err(\theta,\calU)}^3}.
\end{equation*}
Note that $\overline{C}$ is finite by Assumption \ref{ass:Lyap} and Lemma \ref{lem:stability}.
\end{itemize}

We have thus proved that there is a constant $C$ such $\abs{ R_k } \leq C \,  \delta_{k+1}^{3/2}$ for $k \geq 0$; it follows that the sum $\BK{ R_0 + \ldots + R_{m-1} } / T_m$ is less than a constant multiple of $\BK{ \delta_{1}^{3/2} + \ldots + \delta_{m}^{3/2} } /T_m$.  Under Assumption \ref{ass:step-sizes}, this upper bound converges to zero as $m \to \infty$, hence the conclusion.
\end{itemize}

This ends the proof of the almost sure weak convergence of $\pi_m$ towards $\pi$.\\

\noindent
{\bf Proof of Equation \eqref{eq.ergodic.unbounded}}.
By assumption we have $|\phi(\theta)| \leq C_p \, V^p(\theta)$ for some constant $C_p>0$ and exponent $p < p_H / 2$.  To show that $\pi_m(\phi) \to \pi(\phi)$ almost surely, we will use Lemma \ref{lem:stability} and the almost sure weak convergence, which guarantees that $\pi_m(\widetilde{\phi}) \to \pi(\widetilde{\phi})$ for a continuous and bounded test function $\widetilde{\phi}$.

For any $t>0$, the set $\Omega_t \eqdef \{\theta : V(\theta) \le t\}$ is compact and Tietze's extension theorem \citep[Theorem $20.4$]{rudin1986real} yields that there exists a continuous function $\widetilde{\phi}_t$ with compact support that agrees with $\phi$ on $\Omega_t$ and such that $\|\widetilde{\phi}_t\|_{\infty} = \sup \{ |\phi(\theta)| : \theta \in \Omega_t\}$.  We can indeed also assume that $|\widetilde{\phi}_t(\theta)| \leq C_p \, V^p(\theta)$. 
Since Lemma \ref{lem:stability} states that $\sup_m \pi_m(V^{p_H/2})$ is almost surely finite, it follows that
\begin{align*}
|\pi_m(\phi) - \pi_m(\widetilde{\phi}_t)|
&\leq
2 \, C_p \, \pi_m(V^p \, \indicator_{V \geq t})
\leq 2 \, C_p \, \frac{\sup_m \pi_m(V^{p_H/2})}{t^{{p_H/2}-p}},
\end{align*}
where the last inequality follows from the fact that for any probability measure $\mu$, exponents $0 < p < q$ and scalar $t>0$ we have $\mu(V^p \, \indicator_{V \geq t}) \leq \mu(V^q \, \indicator_{V \geq t}) / t^{q-p}$.  
Similarly 
$$|\pi(\phi) - \pi(\widetilde{\phi}_t)| \leq 2 \, C_p \, \pi(V^{p_H/2}) / t^{{p_H/2}-p}.$$ 
By the triangle inequality, we thus have,
\begin{equation*}
|\pi_m(\phi) - \pi(\phi)| \le 2 \, C_p \, \frac{\sup_m \pi_m(V^{p_H/2})}{t^{{p_H/2}-p}}
+
\big|\pi_m(\widetilde{\phi}_t)-\pi(\widetilde{\phi}_t) \big|
+
2 \, C_p \, \frac{\pi(V^{p_H/2})}{t^{{p_H/2}-p}}.
\end{equation*}
On the right-hand-side, the term in the middle can be made arbitrarily small as $m \to \infty$ since $\pi_m$ converges weakly towards $\pi$, while the other two terms converges to zero as $t \to \infty$. This concludes the proof of Equation \eqref{eq.ergodic.unbounded}.\\

\noindent
{\bf Proof of Equation \eqref{eq.weighted.LLN}}. 
The approach is very similar to the proof of Equation \eqref{eq.ergodic.unbounded} and for this reason we only highlight the main differences. The same argument shows that the sequence $\pi^{\omega}_m$ is tight and it suffices to show that $\pi^{\omega}_{\infty}(\A \phi) = 0$ for any weak limit $\pi^{\omega}_{\infty}$ of the sequence $(\pi^{\omega}_m )_{m \geq 0}$ for obtaining the almost sure weak convergences of $( \pi^{\omega}_m )_{m \geq 0}$ towards $\pi$.  One can then upgrade this almost sure weak convergence to a Law of Large Numbers. To prove \eqref{eq.weighted.LLN}, we thus concentrate on proving that $\pi^{\omega}_{\infty}(\A \phi) = 0$. For a smooth and compactly supported test function $\phi$ we use the decomposition $\pi^{\omega}_m(\A \phi)=S_1(m) + S_2(m) + S_3(m)$ with
\begin{align*}
\left\{
\begin{array}{ll}
S_1(m)&=\frac{1}{\Omega_m}\sum_{k=1}^m \frac{\omega_k }{\delta_k}  \big(\EE_{k-1}[\phi(\theta_k)]-\phi(\theta_{k})\big)\\
S_2(m)&=\frac{1}{\Omega_m}\sum_{k=1}^m \frac{\omega_k }{\delta_k}  \big(\phi(\theta_k) - \phi(\theta_{k-1})\big)\\
S_3(m)&=\pi^{\omega}_m(\A \phi)  -  \frac{1}{\Omega_m} \sum_{k=1}^m \frac{\omega_k}{\delta_k} \EE_{k-1}[\phi(\theta_k)-\phi(\theta_{k-1})]
\end{array}
\right.
\end{align*}
and prove that each term converges to zero almost surely.
For $S_1(m)$, by Lemma \ref{lem.MR} it suffices to show that $\sum_{k \geq 1} (\omega_k / \delta_k)^2 \, \EE \sqBK{ \curBK{\EE_{k-1}[\phi(\theta_k)]-\phi(\theta_k)}^2 } / \Omega_k^2$ is finite. This follows from the  bound $\EE \sqBK{ \big(\EE_{k-1}[\phi(\theta_k)]-\phi(\theta_k)\big)^2} \lesssim \delta_k$ and the fact that $\sum_{m \geq 0} \omega_m^2 / (\Omega_m^2 \, \delta_m)$ is finite.
For $S_2(m)$, we can write it as
\begin{equation*}
S_2(m) = \frac{ -\frac{\omega_1}{\delta_1}\phi(\theta_0) + \frac{\omega_{m+1}}{\delta_{m+1}}\phi(\theta_m) - \sum_{k=1}^m \phi(\theta_k) \, \Delta (\omega_k / \delta_k)}{ \Omega_m }.
\end{equation*}
Because $\Omega_m \to \infty$, $(\omega_{m+1} / \delta_{m+1}) / \Omega_m \to 0$ and $\phi$ is bounded, one can concentrate on proving that $\Omega_m^{-1} \sum_{k=1}^m \phi(\theta_k) \, \Delta (\omega_k / \delta_k)$ converges almost surely to zero. By Lemma \ref{lem.MR}, it suffices to verify that $\sum_{k \geq 1} \EE \sqBK{ \abs{ \phi(\theta_k) \, \Delta (\omega_k / \delta_k) }} / \Omega_k$ is finite; this directly follows from the boundedness of $\phi$ and Assumption \ref{ass:step-sizes:weighted}.
Finally, algebra shows that $S_3(m) = \Omega_m^{-1} \, \sum_{1}^m (\omega_k / \delta_k) \, R_{k-1}$ with the quantity $R_k$ defined in Equation \eqref{eq.R}. It has been proved that there is a constant $C$ such that, almost surely, $|R_k| \leq C \, \delta_{k+1}^{3/2}$ for all $k \geq 0$. Since $\delta_m \to 0$, the rescaled sum $\Omega_m^{-1} \, \sum_{k \leq m} \omega_k  \delta_{k}^{1/2}$ converges to zero as $m \to \infty$.  It follows that $S_3(m)$ converges almost surely to zero.
\end{proof}

%
%
\section{Fluctuations, Bias-Variance Analysis, and Central Limit Theorem} \label{sec.bias.variance}

The previous section shows that, under suitable conditions, for a test function $\phi:\RR^d\to\RR$ the quantity $\pi_m(\phi)$ converges almost surely to $\pi(\phi)$ as $m \to \infty$.  
In this section, we investigate the fluctuations of $\pi_m(\phi)$ around its asymptotic value $\pi(\phi)$.
We establish that the asymptotic bias-variance decomposition of the SGLD algorithm is dictated by the behaviour of the sequence
\begin{align}
\biasvariance_m  \eqdef  
T_m^{-1/2} \, \sum_{k=0}^{m-1} \, \delta^2_{k+1}.
\label{eq.bias.variance.ratio}
\end{align}
Indeed, the proof of Theorem \ref{thm.clt} reveals that the fluctuations of $\pi_m(\phi)$  are of order $\mathcal{O}\BK{ T_m^{-1/2} }$ and its bias is of order $\mathcal{O}\BK{ T_m^{-1} \sum_{k=0}^{m-1} \delta_{k+1}^2 }$; the quantity $\biasvariance_m$ is thus the ratio of the typical scales of the bias and fluctuations.
In the case where $\biasvariance_m \to 0$, the fluctuations dominate the bias and the rescaled difference $T^{1/2}_m \times \BK{ \pi_m(\phi) - \pi(\phi)}$ converges weakly to a centred Gaussian distribution. In the case where $\biasvariance_m \to \biasvariance_\infty \in (0,\infty)$, there is an exact balance between the scale of the bias and the scale of the fluctuations; the rescaled quantity $T^{1/2}_m \times \BK{ \pi_m(\phi) - \pi(\phi)}$ converges to a non-centred Gaussian distribution. Finally, in the case where $\biasvariance_m \to \infty$, the bias dominates and the rescaled quantity $\BK{T_m^{-1}\sum_{k=1}^{m} \delta_k^2}^{-1} \times \BK{ \pi_m(\phi) - \pi(\phi)}$ converges in probability to a quantity $\mu(\phi) \in \RR$ whose exact value is described in the sequel. The strategy of the proof is standard;  the solution $h$ of the Poisson equation
\begin{equation} \label{eq.Poisson}
\phi - \pi(\phi) = \A h
\end{equation}
is introduced so that the additive functional $\pi_m(\phi)$ of the trajectory of the Markov process $\{\theta_k\}_{k \geq 0}$ can be expressed as the sum of a martingale and a remainder term. A central limit for martingales can then be invoked to describe the asymptotic behaviour of the fluctuations

%
\begin{theorem} \label{thm.clt} {\bf (Fluctuations)}
Let the step-sizes $(\delta_m)_{m \geq 1}$ satisfy Assumption  \ref{ass:step-sizes} and assume that Assumption \ref{ass:Lyap} 
holds for an exponent $p_H \geq 5$.
Let $\phi:\RR^d \to \RR$ be a test function and assume that the unique solution $h:\RR^d \to \RR$ to the Poisson Equation \eqref{eq.Poisson} satisfies  $\|\nabla^{n}h(\theta)\| \lesssim V^{p_H}(\theta)$ for $n\le 4$ and has a bounded fifth derivative. Define $\sigma^2(\phi) = \pi\BK{\norm{\nabla h}^2}$.
\begin{itemize}
%
%
\item
In  case  the fluctuations dominate, i.e.\ $\biasvariance_m \to 0$, the following convergence in distribution holds,
\begin{equation} \label{eq.clt.unbiased}
\lim_{m \to \infty}\; T^{1/2}_m \, \big\{ \pi_m(\phi) - \pi(\phi) \big\}
\;=\;
\Normal\big(0, \sigma^2(\phi) \big).
\end{equation}
%
%
\item
In  case  the fluctuations and the bias are on the same scale, i.e.\ $\biasvariance_m \to \biasvariance_{\infty} \in (0,\infty)$, the following convergence in distribution holds,
\begin{equation} \label{eq.clt.biased}
\lim_{m \to \infty}\; T^{1/2}_m \, \big\{ \pi_m(\phi) - \pi(\phi) \big\}
\;=\;
\Normal\big(\mu(\phi), \sigma^2(\phi) \big),
\end{equation}
with the asymptotic bias
$$\mu(\phi) = -\biasvariance_{\infty}  \EE\left[ \frac{1}{8} \nabla^2 h(\Theta) \widehat{\nabla \log \pi}(\Theta,\calU)^2 +  \frac{1}{4} \nabla^3 h(\Theta)  \nabla \log \pi(\Theta) +  \frac{1}{24}\nabla^4 h(\Theta) \right]$$ where the  random variables $\Theta \dist \pi$ and $\calU$ are independent.
%
%
%
\item
In  case  the bias dominates, i.e.\ $\biasvariance_m \to \infty$, the following limit holds in probability,
\begin{equation} \label{eq.cv.proba}
\lim_{m \to \infty} \; \frac{ \pi_m(\phi) - \pi(\phi)}{T_m^{-1}\sum_{k=1}^{m} \delta_k^2}
\;=\; \mu(\phi) .
\end{equation}
\end{itemize}
\end{theorem}
%
%
\begin{proof}
The proof follows the strategy described in \cite{lamberton2002recursive}, with the additional difficulty that only unbiased estimates of the drift term of the Langevin diffusion are available. We use the decomposition
\begin{align} \label{eq.decomposition.clt}
	\pi_m(\phi)-\pi(\phi) 
	&=
	\curBK{ \frac{\sum_{k=0}^{m-1} \delta_{k+1} \, \A h(\theta_k)
	- \BK{ h(\theta_{k+1})- h(\theta_k) }}{T_m} }
	+
	\curBK{ \frac{h(\theta_m) - h(\theta_0)}{T_m} }.
\end{align}
A fifth order Taylor expansion and Equation \eqref{eq.sgld} yields that
\begin{align} 
	h(\theta_{k+1})- h(\theta_k) 
	&=
	\sum_{n=1}^4 \curBK{ \sum_{i=0}^n \calC^{(k)}_{n,i} \, \delta_{k+1}^{(n+i)/2} }
	+
	\nabla^5 h(\xi_k) \, \BK{ \theta_{k+1} - \theta_k}^5 / 5!.
\end{align}
In the above, we have defined  $\calC^{(k)}_{n,i} \equiv \BK{ 2^i \, i! \, (n-i)! }^{-1} \, \nabla^n h(\theta_k)  \widehat{\nabla \log \pi}(\theta_k,\calU_{k+1})^i \eta_{k+1}^{n-i}$; the quantity $\xi_k$ lies between $\theta_k$ and $\theta_{k+1}$. It follows from the expression \eqref{eq.generator} of the generator of the $\A$ of the Langevin diffusion \eqref{eq:overdampedLangevin} and decomposition \eqref{eq.decomposition.clt} that $\pi_m(\phi)-\pi(\phi) = \scrF_m + \scrB_m + \scrR_m$ where the fluctuation and bias terms are given by
\begin{align*} 
	\scrF_m &\equiv 
	-\frac{1}{T_m} \sum_{k=0}^{m-1}  \calC^{(k)}_{1,0} \delta_{k+1}^{1/2}
	\quad \textrm{and} \quad
	\scrB_m \equiv 
	-\frac{1}{T_m} \sum_{k=0}^{m-1} \curBK{ \, \calC^{(k)}_{2,2} + \calC^{(k)}_{3,1} + \calC^{(k)}_{4,0} \, } \delta_{k+1}^{2}
\end{align*}
while the remainder term reads
\begin{equation} \label{eq.remainder.term} 
\begin{aligned}
	\scrR_m 
	&\equiv 
	-\frac{1}{T_m} \sum_{k=0}^{m-1} \curBK{
	\frac12 \, H(\theta_k, \calU_{k+1})\, \nabla h(\theta_k) 
	+
	\frac12 \, \BK{\eta^2_{k+1} - 1} \, \nabla^2 h(\theta_k) } \, \delta_{k+1} \\
	&\quad -\frac{1}{T_m} \sum_{k=0}^{m-1} \curBK{
	\sum_{(n,i) \in \cali_{\scrR}} \calC^{(k)}_{n,i} \, \delta_{k+1}^{(n+i)/2} }
	-\frac{1}{T_m} \sum_{k=0}^{m-1} \nabla^5 h(\xi_k) \, \BK{ \theta_{k+1} - \theta_k}^5 / 5! \\
	&\quad + \curBK{ \frac{h(\theta_m) - h(\theta_0)}{T_m} }
\end{aligned}
\end{equation}
for $\cali_{\scrR} = \bigcup_{p \in \{3,5,6,7,8\}} \cali_{\scrR,p}$ and  $\cali_{\scrR,p} \equiv \curBK{ (n,i) \in [1:4] \times [0:4] \, : \, i \leq n, \, i+n = p }$. We will show that the remainder term is negligible in the sense that each term on the R.H.S of Equation \eqref{eq.remainder.term}, when multiplied by either $T_m^{1/2}$ or $T_m(\sum_{k=0}^{m-1} \delta_{k+1}^2)^{-1}$, converges in probability to zero; in other words, each one of these terms is dominated asymptotically by either the fluctuations or the bias and is thus negligible. We then show that when multiplied by $T_m^{1/2}$, the fluctuation term converges in distribution to $\Normal(0,\sigma^2(\phi))$.  Finally, we show that the bias term converge to $\mu(\phi)$ when rescaled by its typical scale, $T_m(\sum_{k=0}^{m-1} \delta_{k+1}^2)^{-1}$.  Putting these results together under the three cases of $\biasvariance_m\to 0$, $\biasvariance_m\to \biasvariance_\infty\in(0,\infty)$ and $\biasvariance_m\to \infty$ leads to the results of the Theorem.\\

%
%
\vspace*{1em}
\noindent
{\bf Remainder term:} we start by proving that the term $\scrR_m$ is negligible. The term $\curBK{h(\theta_m)-h(\theta_0)} / T_m^{1/2}$ converges to zero in probability because $\abs{ h(\theta)} \lesssim V^{p_H}(\theta)$ and Lemma \ref{lem:stability} shows that $\sup_{m \geq 0} \, \EE[V^{p_H}(\theta_m)]$ is almost surely finite. Similarly, Assumptions \ref{ass:step-sizes} and  \ref{ass:Lyap} and Lemma \ref{lem:stability} yield that
\begin{align*}
	\E{\nabla^5 h(\xi_k) \, \BK{ \theta_{k+1} - \theta_k}^5}
	\lesssim
	\E{ \abs{\eta_{k+1}}^5 } \, \delta_{k+1}^{5/2} 
	+ \E{ \abs{\widehat{\nabla\log\pi}(\theta_k,\calU_{k+1})} ^5} \delta_{k+1}^5 \lesssim \delta_{k+1}^{5/2}
\end{align*}
from which it follows that $\curBK{ \sum_{k=0}^{m-1} \nabla^5 h(\xi_k) \, \BK{ \theta_{k+1} - \theta_k}^5} / \curBK{\sum_{k=0}^{m-1} \delta_{k+1}^2}$ converges to zero in probability; we have exploited the fact that $\nabla^5 h$ is assumed to be globally bounded. Essentially the same argument yield that the high-order terms are asymptotically negligible: for $(n,i) \in \cali_{\scrR,p}$ and $p \in \{5,6,7,8\}$ the limit
\begin{equation*}
	\lim_{m \to \infty} \; \frac{\sum_{k=0}^{m-1} \calC^{(k)}_{n,i} \, \delta_{k+1}^{(n+i)/2}}{\sum_{k=0}^{m-1} \delta_{k+1}^2} \; = \; 0
\end{equation*}
holds in probability because the coefficients $\calC^{(k)}_{n,i}$ are uniformly bounded in expectation and the quantity $\BK{ \sum_{k=0}^{m-1}  \delta_{k+1}^{(n+i)/2}}/\BK{ \sum_{k=0}^{m-1} \delta_{k+1}^2}$ converges to zero since $(n+i)/2 \geq 5/2$ and $\delta_k\to 0$. To conclude, one needs to verify that the low order terms are also negligible in the sense that the limit
\begin{align*}
	\lim_{m \to \infty} \; \frac{ \sum_{k=0}^{m-1} X^{(k)}_{n,i} \delta_{k+1}^{(n+i)/2}}{T_m^{1/2}} = 0 
\end{align*}
holds in probability with 
$X^{(k)}_{1,1} = \nabla h(\theta_k) \, H(\theta_k,\calU_{k+1})$ and 
$X^{(k)}_{2,0} = \nabla^2 h(\theta_k) (\eta_{k+1}^2-1)$ and 
$X^{(k)}_{2,1} = -\calC^{(k)}_{2,1}$ and 
$X^{(k)}_{3,0} = -\calC^{(k)}_{3,0}$. Since $\CE{ X^{(k)}_{n,i} }{\mathcal{F}_k} = 0$ where $\mathcal{F}_k = \sigma \BK{\theta_0, \ldots, \theta_k}$ is the natural filtration associated to the process $\BK{\theta_k}_{k \geq 0}$ it follows that 
\begin{align*}
	\EE \sqBK{ \BK{ \frac{\sum_{k=0}^{m-1} X^{(k)}_{n,i} \delta_{k+1}^{(n+i)/2}}{T_m^{1/2} } }^2 }
	= \frac{ \sum_{k=0}^{m-1}  \EE \sqBK{ (  X^{(k)}_{n,i})^2 } \, \delta_{k+1}^{n+i}}{T_m}
	\lesssim \frac{\sum_{k=0}^{m-1} \delta_{k+1}^{n+i}}{T_m} \to 0.
\end{align*}
We made use of the fact that the expectations $\EE \sqBK{ (X^{(k)}_{n,i})^2}$ are uniformly bounded for all $k\ge 0$ by the same arguments as above, and that the final expression converges to 0 since $n+i\ge 2$, $\delta_m\to 0$ and $T_m\to \infty$. This concludes the proof that the remainder term $\scrR_m$ is asymptotically negligible.\\

%
%
\vspace*{1em}
\noindent {\bf Fluctuation term:} we now prove that the fluctuations term converges in distribution at Monte-Carlo rate towards a Gaussian distribution, %
\begin{align*}
	T_m^{1/2}  \, \scrF_m
	\equiv 	
	-\frac{\sum_{k=0}^{m-1} \nabla h(\theta_k) \, \delta^{1/2}_{k+1} \, \eta_{k+1}}{T_m^{1/2}}
	\to \Normal \BK{ 0,\sigma^2(\phi) }.
\end{align*}
Using the standard martingale central limit theorem (e.g. Theorem $3.2$, Chapter $3$ of \citep{hall1980martingale}), it suffices to verify that for any $\epsilon>0$ the following limits hold in probability,
\begin{align*}
	\lim_{m \to \infty} 
	\sum_{k=0}^{m-1} \frac{\EE_k \sqBK{ Z_k^2 \, \mathbb{I}\BK{Z_k^2 > T_m\epsilon} } }{T_m} = 0
	\quad \textrm{and} \quad 
	\lim_{m \to \infty} 
	\frac{ \sum_{k=0}^{m-1} \EE_k \sqBK{  Z_k^2 } }{ T_m } = \sigma^2(\phi)
\end{align*}
with $Z_k \eqdef \nabla h(\theta_k) \, \delta^{1/2}_{k+1} \, \eta_{k+1}$. Since $\EE_k \sqBK{ Z_k^2 } = \nabla h(\theta_k)^2 \, \delta_{k+1}$ and the function $\theta \mapsto \nabla h(\theta)^2$ satisfies the assumptions of Theorem \ref{thm.LLN}, the second limit directly follows from Theorem \ref{thm.LLN}.  For proving the first limit, note that the Cauchy-Schwarz's inequality and the boundedness of $\nabla h$ imply that $\EE_k \sqBK{ Z_k^2 \, \mathbb{I}\BK{Z_k^2 > T_m\epsilon}} \lesssim \delta_{k+1} \times \PP \sqBK{ \delta_{k+1} \, \norm{ \nabla h }^2_{\infty} \, \eta_{k+1}^2 > T_m \, \epsilon}^{1/2}$; the Markov's inequality thus yields that
\begin{equation*}
\sum_{k=0}^{m-1} \EE_k\big[ Z_k^2 \, I\big(Z_k^2 > T_m\epsilon\big)\big] / T_m
\lesssim
\frac{ \sum_{k=0}^{m-1} \delta_{k+1}^2 } { T_m^2 \, \epsilon}.
\end{equation*}
Since $T_m^{-2} \, \sum_{k=0}^{ m-1} \delta_{k+1}^2 \to 0$, the conclusion follows.\\

%
%
\vspace*{1em}
\noindent
{\bf Bias term:} we conclude by proving that the bias term is such that the limit
\begin{align*}
	\lim_{m \to \infty} \; \frac{\scrB_m}{\sum_{k=1}^{m} \delta_k^2 / T_m} \,  \; = \; \mu(\phi)
\end{align*}
holds in probability. The quantity $\scrB_m / \BK{ \sum_{k=1}^{m} \delta_k^2 / T_m}^{-1}$ can also be expressed as
\begin{align} \label{eq.bias.cv}
	\frac{\sum_{k=0}^{m-1} \Psi(\theta_k) \, \delta_{k+1}^2}{\sum_{k=0}^{m-1} \delta_{k+1}^2} 
	+ \frac{\sum_{k=0}^{m-1} \Delta M_k \, \delta_{k+1}^2}{\sum_{k=0}^{m-1} \delta_{k+1}^2}
\end{align}
for a martingale difference term $\Delta M_k \equiv \BK{\calC^{(k)}_{2,2} + \calC^{(k)}_{3,1} + \calC^{(k)}_{4,0} } - \Psi(\theta_k)$ where $\Psi(\theta_k) \equiv \CE{\calC^{(k)}_{2,2} + \calC^{(k)}_{3,1} + \calC^{(k)}_{4,0} }{\mathcal{F}_k}$ and $\BK{\calC^{(k)}_{2,2} + \calC^{(k)}_{3,1} + \calC^{(k)}_{4,0}}$ equals
\begin{align*}
	\frac{1}{8} \nabla^2 h(\theta_k) \widehat{\nabla \log \pi}(\theta_k,\calU_{k+1})^2 
	+ \frac{1}{4} \nabla^3 h(\theta_k) \widehat{\nabla \log \pi}(\theta_k,\calU_{k+1})\eta_{k+1}^2 
	\quad + \frac{1}{24} \nabla^4 h(\theta_k) \eta_{k+1}^4.
\end{align*}
Under the assumptions of Theorem \ref{thm.clt}, the function $\Psi$ satisfies the hypothesis of Theorem \ref{thm.LLN} applied to the weight sequence $\{\delta_k^2\}_{k \geq 0}$; it follows that the first term in Equation \eqref{eq.bias.cv} converge almost surely to $\mu(\phi)$. It remains to prove that the second term in Equation \eqref{eq.bias.cv} also converges almost surely to zero. By Lemma \ref{lem.MR}, it suffices to prove that the martingale 
\begin{align*}
	m \mapsto \sum_{k=0}^m \frac{\Delta M_k \, \delta_{k+1}^2}{\sum_{j=1}^{k+1} \delta_{j+1}^2}
\end{align*}
is bounded in $L^2$. Under the Assumption of Theorem \ref{thm.clt}, Lemma \ref{lem:stability} yields that the martingale difference term $\Delta M_k$ is uniformly bounded in $L^2$ from which the conclusion readily follows.
\end{proof}

For the standard choice of step-sizes $\delta_m = (m_0 + m)^{-\alpha}$ the statistical fluctuations dominate in the range $1/3 < \alpha \leq 1$, there is an exact balance between bias and fluctuations for $\alpha=1/3$, and the bias dominates for $0 < \alpha < 1/3$.  The optimal rate of convergence is obtained for $\alpha = 1/3$ and leads to an algorithm that converges at rate $m^{-1/3}$.

%
%
\section{Diffusion limit}
\label{sec.diff.lim}
In this section we show that, when observed on the right (inhomogeneous) time scale, the sample path of the SGLD algorithm converges to the continuous time Langevin diffusion of Equation \eqref{eq:overdampedLangevin}, confirming the heuristic discussion in \citet{welling2011bayesian}.

The result is based on the continuity properties of the It{\^o}'s map $\ito: \continuous([0,T], \RR^d) \to \continuous([0,T], \RR^d)$, which sends a continuous path $w \in \continuous([0,T], \RR^d)$ to the unique solution $v = \ito(w)$ of the integral equation,
\begin{equation*}
\label{eq.ito.integral.equation}
v_t = \theta_0 + 
\frac12 \, \int_{s=0}^t \, \nabla \log \pi(v_s) \, ds + 
w_t
\qquad \textrm{for all} \quad
t \in [0,T].
\end{equation*}
If the drift function $\theta \mapsto \frac12  \nabla \log \pi(\theta)$ is globally Lipschitz, then the It{\^o}'s map $\ito$ is well defined and continuous.  Further, the image $\ito(W)$ under the It\^o map of a standard Brownian motion $W$ on $[0,T]$ can be seen to be described by Langevin diffusion \eqref{eq:overdampedLangevin}.  

The approach, inspired by ideas in \citet{mattingly2012diffusion,pillai2012optimal}, is to construct a sequence of coupled Markov chains $(\theta^{(r)})_{r\ge 1}$, each started at the same initial state $\theta_0\in \RR^d$ and evolved according to the SGLD algorithm with step-sizes $\delta^{(r)}\eqdef (\delta^{(r)}_k)_{k=1}^{m(r)}$ such that $$\sum_{k=1}^{m(r)} \delta^{(r)}_k=T$$ and with increasingly fine mesh sizes
$\mesh(\delta^{(r)}) \to 0$ with
\begin{equation*}
	\mesh(\delta^{(r)})
	\eqdef
	\max \curBK{ \delta^{(r)}_k \;:\; 1 \leq k \leq m(r) }.
\end{equation*}
Define $T^{(r)}_0= 0$ and $T^{(r)}_k= \delta^{(r)}_1+\cdots+\delta^{(r)}_k$ for each $k\ge 1$.  
The Markov chains are coupled to $W$ as follows:
\begin{align} \label{eq.seq.MC}
\left\{
\begin{array}{ll}
\noise^{(r)}_{k} &= (\delta^{(r)}_k)^{-1/2} \BK{ W(T^{(r)}_k)-W(T^{(r)}_{k-1}) } \\
\theta^{(r)}_{k} &= \theta^{(r)}_{k-1} + 
\frac12 \, \delta^{(r)}_{k}  \curBK{ \nabla \log \pi(\theta^{(r)}_{k-1}) + \err(\theta^{(r)}_{k-1}, \calU^{(r)}_{k}) \, }
\,+\,
(\delta^{(r)}_{k})^{1/2} \, \noise^{(r)}_{k},
\end{array}
\right.
\end{align}
for an i.i.d.\ collection of auxiliary random variables $(\calU^{(r)}_k)_{r\ge 1, k \geq 1}$.  Note that  $(\noise^{(r)}_k)_{k\ge 1}$ form an  i.i.d.\ sequence of $\Normal(0,1)$ variables for each $r$.
We can construct piecewise affine continuous time sample paths $(S^{(r)})_{r \ge 1}$ by linearly interpolating the Markov chains,
\begin{equation} \label{eq.interpolation.S}
S^{(r)} \BK{ x T^{(r)}_{k-1} + (1-x) T^{(r)}_{k} } = 
x \, \theta^{(r)}_{k-1} + (1-x) \, \theta^{(r)}_{k},
\end{equation}
for  $x\in[0,1]$.  The approach then amounts to showing that each $S^{(r)}$ can be expressed as $\ito(\widetilde{W}^{(r)})+e^{(r)}$, where $\widetilde{W}^{(r)}$ is a sequence of stochastic processes converging to $W$ and $e^{(r)}$ is asymptotically negligible, and making use of the continuity properties of the It\^o map $\ito$.

%
%
\begin{theorem} \label{thm.diff.lim}
Let Assumption \ref{ass:Lyap} holds and suppose that the drift function $\theta \mapsto  (1/2) \nabla \, \log \pi(\theta)$ is globally Lipschitz on $\RR^d$.  If $\mesh(\delta^{(r)})\to 0$ as $r\to\infty$, then the sequence of continuous time processes $(S^{(r)})_{r\ge 1}$ defined in Equation \eqref{eq.interpolation.S} converges weakly on $\big( \continuous([0,T], \RR^d), \| \cdot \|_{\infty} \big)$ to the Langevin diffusion \eqref{eq:overdampedLangevin} started at $S_0=\theta_0$.
\end{theorem}
%
%
\begin{proof}
Since the drift term $s \mapsto (1/2) \, \nabla \log \pi(s)$ is globally Lipschitz on $\RR^d$,  Lemma $3.7$ of \citep{mattingly2012diffusion} shows that the It{\^o}'s map $\ito: \continuous([0,T], \RR^d) \to \continuous([0,T], \RR^d)$ is well-defined and continuous, under the topology over the space $\continuous([0,T], \RR^d)$ induced by the supremum norm $\|w\|_{\infty} \equiv \sup \{ |w_t| : 0 \leq t \leq T \}$. By the Continuous Mapping Theorem, because the Langevin diffusion \eqref{eq:overdampedLangevin} can be seen as the image under the It{\^o}'s map $\ito$ of a standard Brownian motion on $[0,T]$ evolving in $\RR^d$,  it suffices to verify that the process $S^{(r)}$ can be expressed as $\ito(\widetilde{W}^{(r)}) + e^{(r)}$ where $\widetilde{W}^{(k)}$ is a sequence of stochastic processes that converge weakly in $\continuous([0,T], \RR^d)$ to a standard Brownian motion $W$ and $e^{(r)}$ is an error term that is asymptotically negligible in the sense that $\|e^{(r)}\|_{\infty}$ converges to zero in probability. 

For convenience, we define $\widetilde{W}^{(r)}$ as the continuous piecewise affine processes that satisfies $\widetilde{W}^{(r)}(T^{(r)}_k) = W(T^{(r)}_k)$ for all $ 0 \leq k \leq m(r)$
and that is affine in between. It follows that for any time $T^{(r)}_{k-1} \leq t \leq T^{(r)}_{k}$ we have
\begin{align*}
S^{(r)}(t)  &= S^{(r)}(T^{(r)}_{k-1})
+
\left(
\int_{T^{(r)}_{k-1}}^{t}
\frac12 \nabla \log \pi\big(S^{(r)}(T^{(r)}_{k-1}) \big)  du
+
\widetilde{W}^{(r)}(t)-\widetilde{W}(T^{(r)}_{k-1})
\right)\\
&\qquad + 
\frac{1}{2} \int_{T^{(r)}_{k-1}}^{t}
\err\big( S^{(r)}(T^{(r)}_{k-1}), \calU^{(r)}_k\big) du\\
&=
\underbrace{\theta_0
+
\left(
\int_{0}^{t}
\frac12 \nabla \log \pi\big(S^{(r)}(u) \big)  du
+
\widetilde{W}^{(r)}(t)\right)}_{\ito(\widetilde{W})(t)}\\
&\qquad
+
\underbrace{\int_{0}^{t}
\frac12 \left(
\nabla \log \pi\big(\widehat{S}^{(r)}(u) \big)
-
\nabla \log \pi\big(S^{(r)}(u) \big)
\right)
  du}_{e^{(r)}_1(t)}\\
&\qquad
+
\frac{1}{2} \underbrace{\int_{0}^{t}
\err\big( \widehat{S}^{(r)}(u), \calU^{(r)}_k\big) du}_{e^{(r)}_2(t)},
\end{align*}
where $\widehat{S}^{(r)}$ is a piecewise constant (non-continuous) process, $\widehat{S}^{(r)}(t) = S^{(r)}(T^{(r)}_{k-1})=\theta^{(r)}_{k-1}$ for $t\in[T^{(r)}_{k-1},T^{(r)}_k)$.  
The process $S^{(r)}$ can thus be expressed as the sum
$
\ito(W^{(r)}) + e^{(r)}_1 + e^{(r)}_2
$.
Since the mesh-size of the partition $\delta^{(r)}$ converges to zero as $r \to \infty$, standard properties of Brownian motions yield that $\widetilde{W}^{(r)}$ converges weakly in $\big(\continuous([0,t], \RR^d), \| \cdot \|_{\infty,[0,T]}\big)$ to $W$, a standard Brownian motion in $\RR^d$. To conclude the proof, we need to check that the quantities $\|e^{(r)}_1\|_{\infty}$ and $\|e^{(r)}_2\|_{\infty}$ converge to zero in probability. 
To prove $\E{ \|  e^{(r)}_2 \|_\infty^2 }\to 0$ in probability, we have,
\begin{align*}
\E{ \|e_2^{(r)}\|_{\infty}^2 }
&\le 4 \, \E{ \|e_2^{(r)}(T)\|^2 }
= 4 \, \sum_{k=1}^{m(r)} \big(\delta^{(r)}_k \big)^2 
\E{\err\BK{ \theta^{(r)}_{k-1}, \calU^{(r)}_k}^2} \\
&\lesssim \sum_{k=1}^{m(r)} \big(\delta^{(r)}_k \big)^2  \E{ V(\theta^{(r)}_{k-1}) }
\leq \mesh(\delta^{(r)})  \, \sum_{k=1}^{m(r)} \delta^{(r)}_k  \E{ V(\theta^{(r)}_{k-1}) } \\
&\leq \mesh(\delta^{(r)}) \times T \times \sup \curBK{ \E{ V(\theta^{(r)}_{k-1}) } \,:\, r\ge 1, 1\le k\le m(r) }
\lesssim \mesh(\delta^{(r)}).
\end{align*}
We have used Doob's martingal inequality, Assumption \ref{ass:Lyap} and Lemma \ref{lem:stability}.
Since $\mesh(\delta^{(r)})$ converges to zero, the conclusion follows.
To prove $\E{ \|e_1^{(r)}\|_{\infty}}\to 0$ in probability, we use Equation \eqref{eq.seq.MC} and note that 
since the drift function $\theta\mapsto \frac{1}{2} \nabla \log \pi(\theta)$ is globally Lipschitz, for each $T^{(r)}_{k-1} \leq u \leq T^{(r)}_k$ we have,
\begin{align*}
\; &\norm{ \nabla \log \big(\widehat{S}^{(r)}(u) \big) - \nabla \log \big(S^{(r)}(u) }
\lesssim \norm{ \theta^{(r)}_{k}-\theta^{(r)}_{k-1} }\\
&\qquad \lesssim
\|\nabla \log \pi(\theta^{(r)}_{k-1}) \| \, \delta^{(r)}_k
+
\|\err\big( \theta^{(r)}_{k-1}, \calU^{(r)}_k\big)\| \, \delta^{(r)}_k
+
\sqrt{\delta^{(r)}_k} \, \| \noise^{(r)}_k \|.
\end{align*}
It follows that 
\begin{align*}
\E{ \|e^{(r)}_1\|_{\infty} } 
\lesssim \sum_{k=1}^{m(r)} \delta^{(r)}_k  \left(
\|\nabla \log \pi(\theta^{(r)}_k) \| \, \delta^{(r)}_k
+
\|\err\big( \theta^{(r)}_k, \calU_k\big)\| \, \delta^{(r)}_k
+
\sqrt{\delta^{(r)}_k} \, \| \noise^{(r)}_k \| \right).
\end{align*}
Since $\mesh(\delta^{(r)})$ converges to zero and by Assumption \ref{ass:Lyap} and Lemma \ref{lem:stability} the suprema
\begin{align*}
\left\{
\begin{array}{ll}
\sup& \curBK{ \EE \big[ \|\nabla \log \pi(\theta^{(r)}_k)  \| \big] \,:\; r \geq 1, 1 \leq k \leq m(r) }, \\
\sup& \curBK{ \EE \big[ \| \err\big( \theta^{(r)}_k, \calU_k\big) \| \big] \,:\; r \geq 1, 1 \leq k \leq m(r) }
\end{array}
\right.
\end{align*}
are finite, it readily follows that $\|e^{(r)}_1\|_{\infty}$ converges to zero in expectation.
\end{proof}

%
%
\section{Numerical Illustrations}
\label{sec.numerics}

In this section we illustrate the use of the SGLD method to a simple Gaussian toy model and to a Bayesian logistic regression problem. We  verify that both models satisfy Assumption \ref{ass:Lyap}, the main assumption needed for our asymptotic results to hold. Simulations are then performed to empirically confirm our theory; for step-sizes sequences of the type $\delta_m=(m_0+m)^{-\alpha}$, both the rate of decay of the MSE and the impact of the sub-sampling scheme are investigated. The main purpose of this article is to establish the missing theoretical foundation of stochastic gradient methods for the approximation of expectations. For more exhaustive simulation studies we refer to   
\cite{welling2011bayesian,AhnKorWel2012,PatTeh2013a,chen2014stochastic}. By considering a logistic regression model, we demonstrate  that the SGLD can be advantageous over the Metropolis-Adjusted-Langevin (MALA) algorithm if the available computational budget only allows a few iterations through the whole data set, see Section \ref{sec:logreg}.

\subsection{Linear Gaussian model}
\label{ex:simpleGaussian}
Consider $N$ independent and identically distributed observations $(x_i)_{i=1}^N$ from the two parameters location model given by
\begin{eqnarray*}
x_{i} \mid \theta \; \sim \; \Normal(\theta,\sigma_{x}^{2}).
\end{eqnarray*}
We use a Gaussian prior $\theta \sim \Normal(0,\sigma_{\theta}^{2})$ and assume that the variance hyper-parameters $\sigma_\theta^2$ and $\sigma_x^2$ are both known.
The posterior density $\pi(\theta)$ is normally distributed with mean $\mu_p$ and variance $\sigma_p^2$ given by
\begin{equation*}
\mu_p = \bar{x} \, \Big(1+\frac{\sigma_x^2}{N \sigma_{\theta}^2} \Big)^{-1}
\qquad \textrm{and} \qquad
\sigma_p^2= \frac{\sigma_x^2}{N} \, \Big(1+\frac{\sigma_x^2}{N \sigma_{\theta}^2} \Big)^{-1}
\end{equation*}
where $\bar{x} = (x_1 + \ldots + x_N)/N$ is the sample average of the observations. In this case, we have
\begin{align*}
\nabla \log \pi(\theta) 
&= -\frac{\theta - \mu_p}{\sigma_p^2}
\quad \textrm{and} \quad 
\err(\theta,\calU)
=
\Big\{ (N/n) \sum_{j \in \mathcal{I}_n(\calU)} x_j
-  \sum_{1 \leq i \leq N} x_i \Big\} / \sigma_x^2
\end{align*}
for a random subset $\mathcal{I}_n(\calU) \subset [N]$ of cardinal $n$.

\subsubsection{Verification of Assumption \ref{ass:Lyap}}
We verify in this section that Assumption \eqref{ass:Lyap} is satisfied for the following choice of Lyapunov function,
\begin{equation*}
V(\theta)=1+\frac{(\theta-\mu_p)^2}{2 \, \sigma_p^2}.
\end{equation*}
Since the error term $H(\theta, \calU)$ is globally bounded, the drift $(1/2) \nabla \log \pi$ and the Lyapunov function $V$ are linear,  Assumptions \eqref{ass:Lyap}.1 and \eqref{ass:Lyap}.2 are satisfied.
Finally, to verify Assumption \eqref{ass:Lyap}.3, it suffices to note that since $\nabla \log \pi(\theta) = -(\theta - \mu_p)/\sigma_p^2$ we have
\begin{align*}
\angleBK{\nabla V(\theta), \frac12 \, \nabla \log \pi(\theta)} 
&= -\frac{(\theta - \mu_p)^2}{2 \, \sigma_p^4} = \frac{1-V(\theta)}{\sigma_p^2}.
\end{align*}
In other words, Assumption \eqref{ass:Lyap}.3 holds with $\alpha=\beta=1/\sigma_p^2$.

\subsubsection{Simulations}\label{sec:SimGauss}

\begin{figure}
\begin{center}
\includegraphics[width=0.95\textwidth]{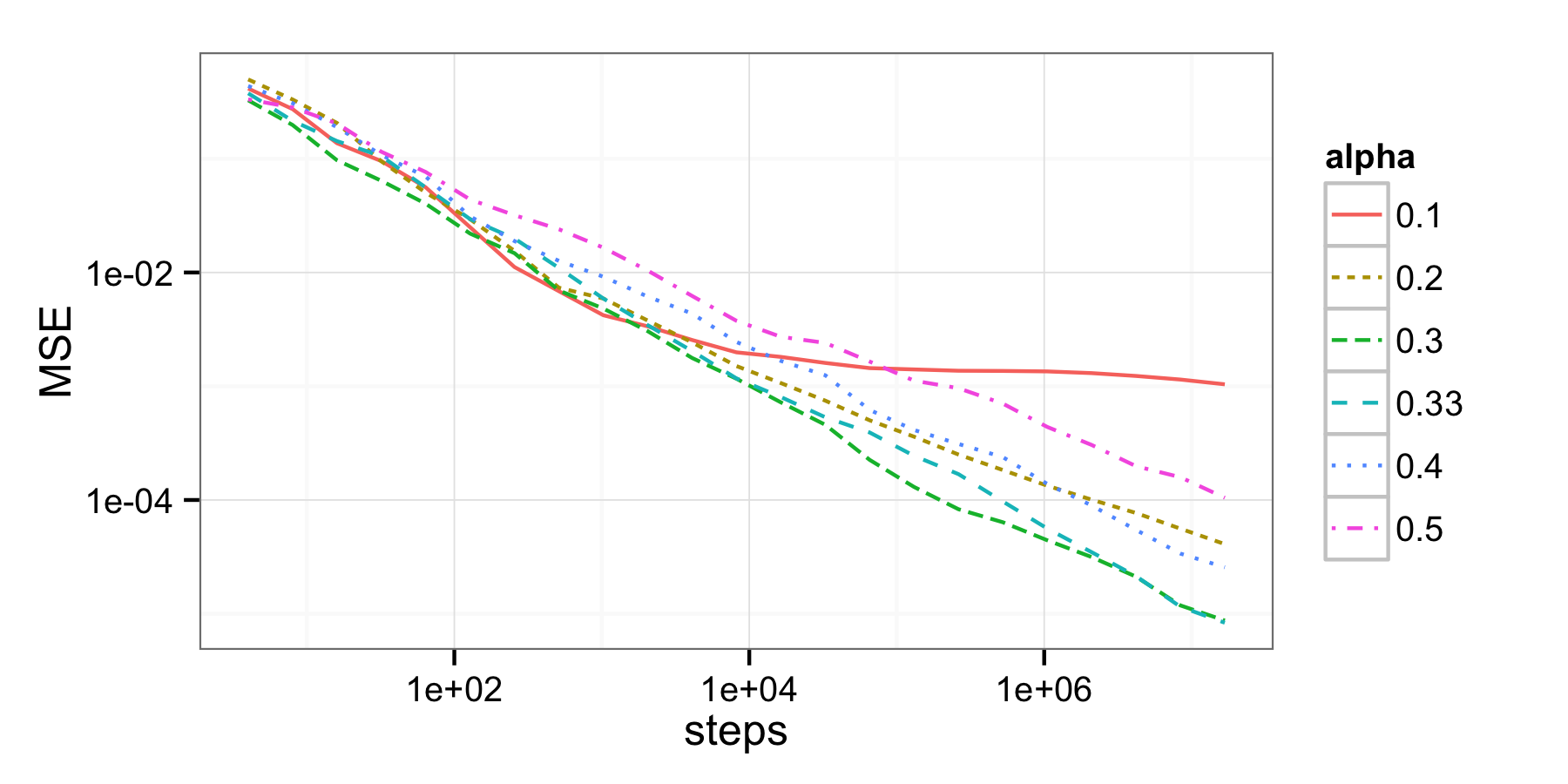}
\end{center}
\caption{Decay of the MSE for step sizes  $\delta_m \asymp m^{-\alpha}$, $\alpha\in \{0.1,0.2,0.3,0.33,0.4,0.5 \}$. The MSE decays algebraically for all step sizes, with fastest decay at approximately $\alpha=0.33$.\\}
\label{fig:MSEdecay}
\end{figure}

\begin{figure}[!h]
\begin{center}
\includegraphics[width=0.80\textwidth]{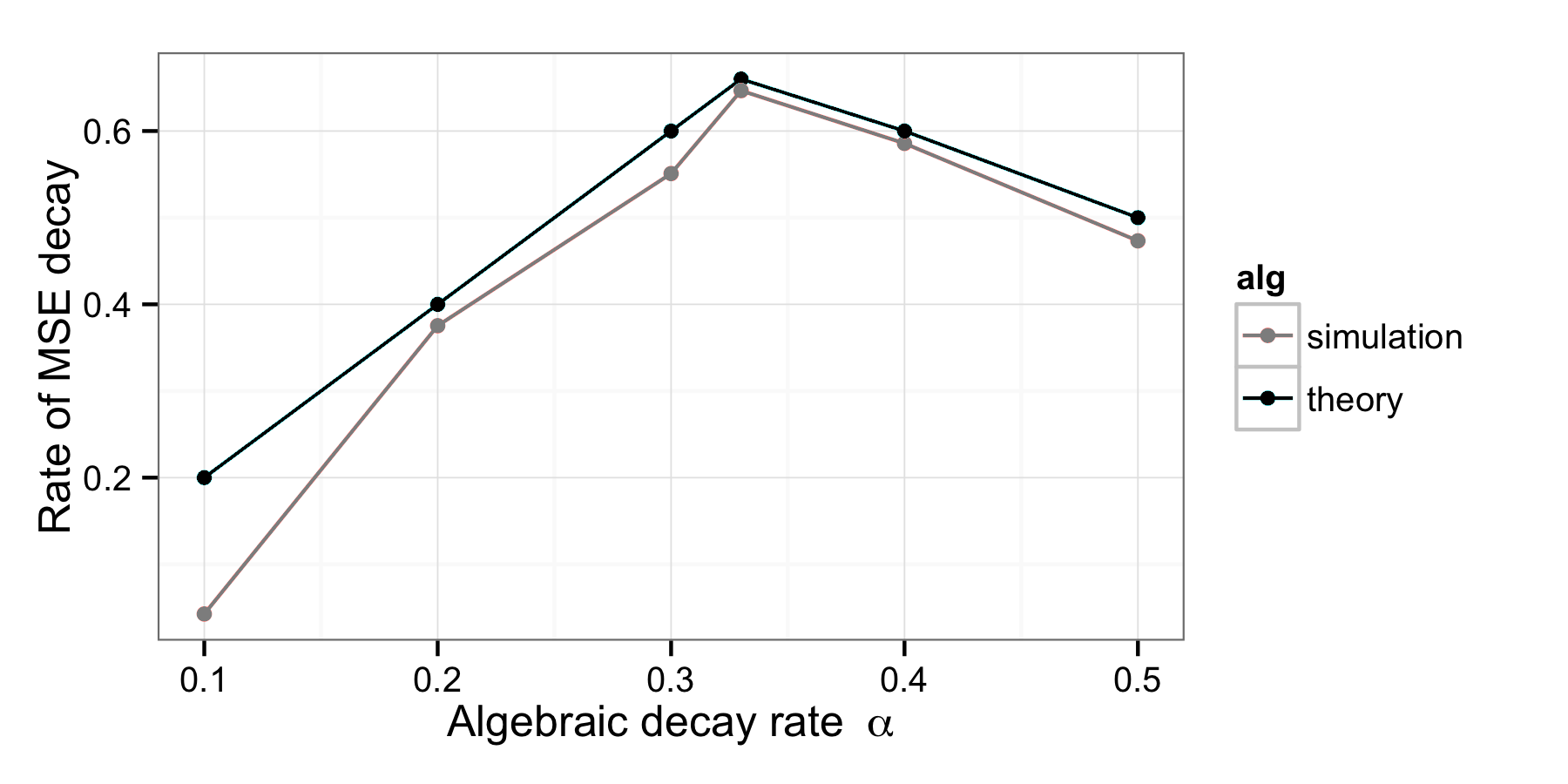}\hspace*{-6em}
\end{center}
\caption{Rates of decay of the MSE obtained from estimating the asymptotic slopes of the plots in Figure \ref{fig:MSEdecay}, compared to theoretical findings of Theorem \ref{thm.clt}.  The fastest convergence rate is achieved at $\alpha=1/3$.  }
\label{fig:MSErate}
\end{figure}
\begin{figure}[!h]
\begin{center}
\includegraphics[width=0.95\textwidth]{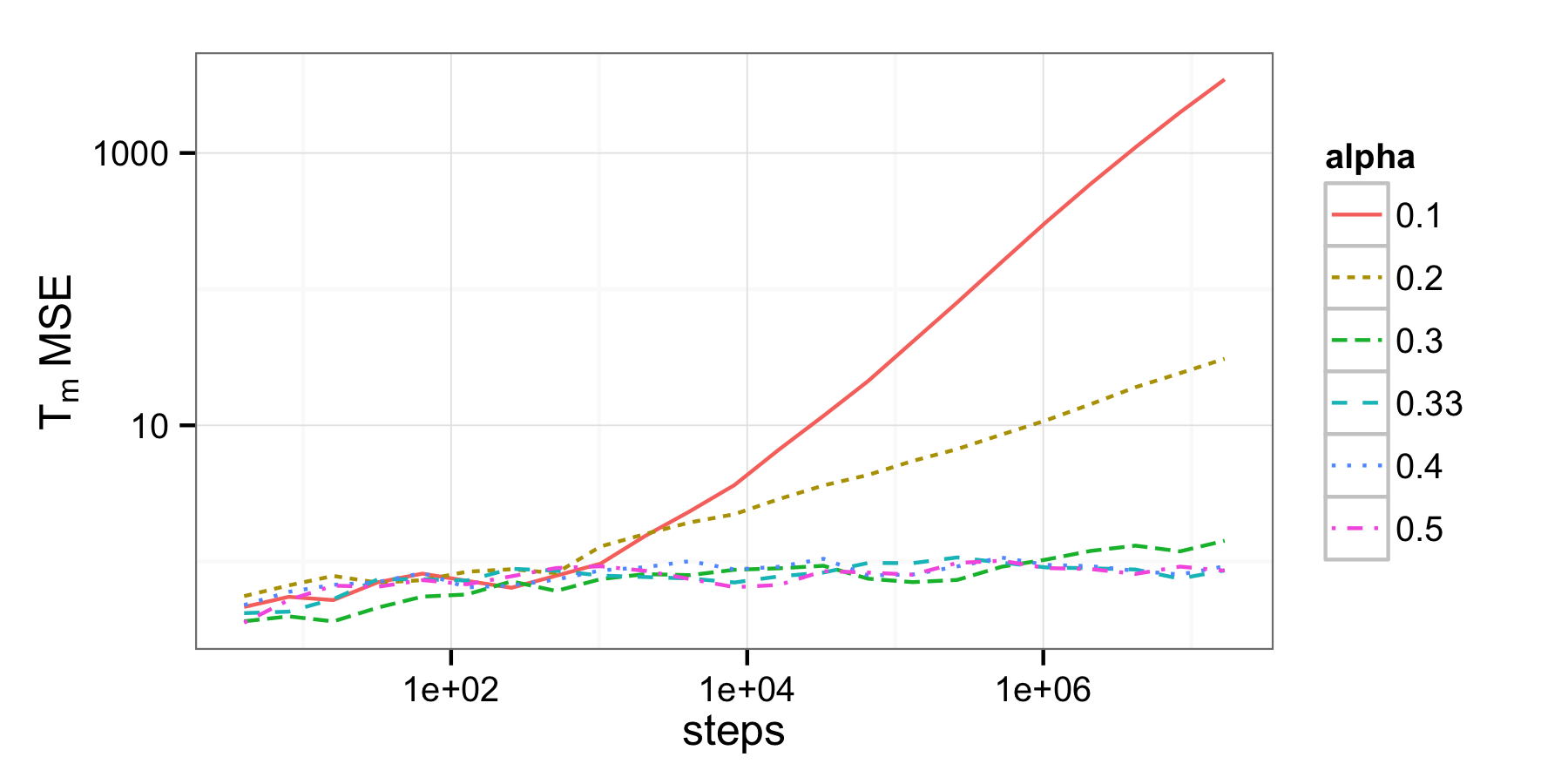}
\end{center}
\caption{Plots of the MSE multiplied by $T_m$ against the number of steps $m$.  The plots are flat for $\alpha\ge 0.33$, demonstrating that the MSE scales as $T_m^{-1}$ in this regime, while the plots diverge for $\alpha<0.33$, demonstrating that it decays at a slower rate here. }
\label{fig:MSEscaling}
\end{figure}

\begin{figure}
\begin{center}
\includegraphics[width=0.95\textwidth]{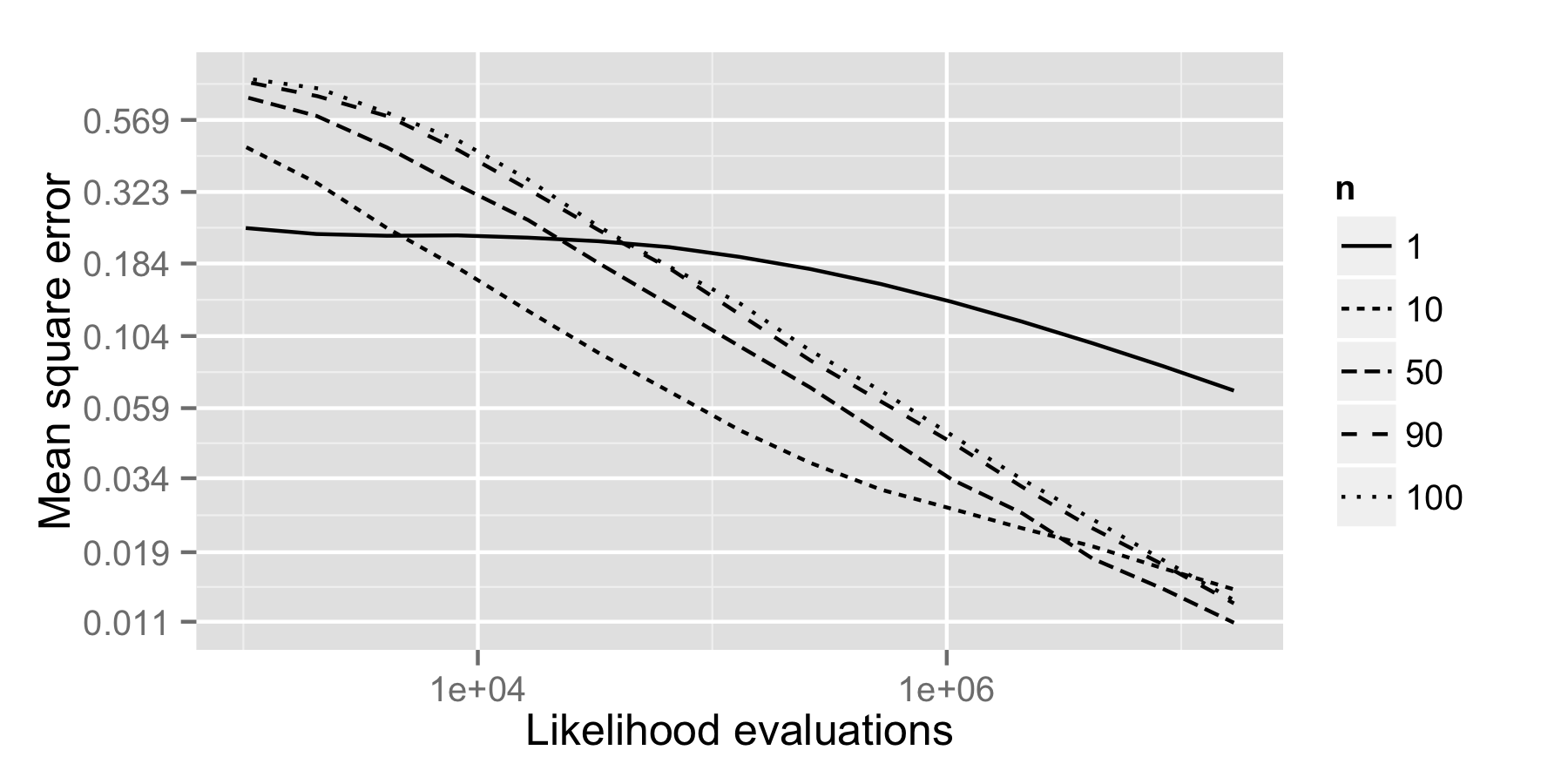}
\end{center}
\caption{Behaviour of the mean squared error for different subsample sizes $n$.\\
}
\label{fig:NonAsympMSE}
\end{figure}

We chose $\sigma_\theta=1$, $\sigma_\data=5$ and created a data set consisting of $\nData=100$ data points simulated from the model. We used $n=10$ as the  size of subsets used to estimate the gradients.  We evaluated the convergence behaviour of SGLD using the test function $\A\varphi$ where  $\phi = \sin\left(x-\mu_p -0.5\sigma_p \right)$.

We are interested in confirming the asymptotic convergence regimes of Theorem \ref{thm.clt} by running SGLD with a range of step sizes, and plotting the mean squared error (MSE) achieved by the estimate $\pi_m(\A\varphi)$ against the number of steps $m$ of the algorithm to determine the rates of convergence.  We used step sizes $\delta_m=(m+m_0(\alpha))^{-\alpha}$, for $\alpha\in\{0.1,0.2,0.3,0.33,0.4,0.5\}$  where $\time_0(\alpha)$ is chosen such  that $\delta_1$ is less than the posterior standard deviation.  According to the Theorem, the MSE should scale as $T_m^{-1}$ for $\alpha> 1/3$, and $\sum_{k=1}^m \delta_k^2/T_m$ for $\alpha\le 1/3$.

The observed MSE is plotted against $m$ on a log-log plot in Figure \ref{fig:MSEdecay}. As predicted by the theory, the optimal rate of decay is around $\alpha_\star = 1/3$. To be more precise, we estimate the rates of decay by estimating the slopes on the log-log plots.  This is plotted in Figure \ref{fig:MSErate}, which also shows a good match to the theoretical rates given in Theorem \ref{thm.clt}, where the best rate of decay is $2/3$ achieved at $\alpha=1/3$.  Finally, to demonstrate that there are indeed two distinct regimes of convergence, in Figure \ref{fig:MSEscaling} we have plotted the MSE multiplied by $T_m$.  For $\alpha>1/3$, the plots remain flat, showing that the MSE does indeed decay as $T_m^{-1}$.  For $\alpha<1/3$, the plots diverge, showing that the MSE decays at a slower rate than $T_m^{-1}$.

%

%
For $\alpha=0.33$, Figure \ref{fig:NonAsympMSE} depicts 
  how the MSE decreases as a function of the number of likelihood evaluations for subsample sizes $\nSubData=1,5,10,50,100$. 
  

\subsection{Logistic Regression}
We verify in this section that Assumption \eqref{ass:Lyap} is satisfied for the following logistic regression model. Consider $N$ independent and identically observations $(y_i)_{i=1}^N$ distributed as
\begin{equation}
\label{eq.logistic}
\mathbb{P}(y_{i} = 1 \mid x_{i},\theta) \; = \; 1 - \mathbb{P}(y_{i} = -1 \mid x_{i},\theta)\; = \;\textrm{logit}\big(\angleBK{\theta, x_{i}} \big)
\end{equation}
for covariate $x_i \in \RR^d$, unknown parameter $\theta \in \RR^d$ and function $\textrm{logit}(z)=e^z / (1+e^z)$. 
We assume a centred Gaussian prior on $\theta \in \RR^d$ with positive definite symmetric covariance matrix $C \in \RR^{d \times d}$. It follows that
\begin{align*}
\nabla \log \pi(\theta)
&=
-C^{-1}\theta + \sum_{i=1}^{N}\text{logit} \big(-y_{i} \angleBK{\theta,x_{i}} \big) \, y_{i} \, x_{i}\\
\err(\theta,\calU)
&=
(N/n)\sum_{j \in \mathcal{I}_n(U)} \text{logit} \big(-y_{j} \angleBK{ \theta,x_{j}} \big) \, y_{j} \, x_{j}
-
\sum_{1 \leq i \leq N}\text{logit} \big(-y_{i} \angleBK{\theta,x_{i}} \big) \, y_{i} \, x_{i}
\end{align*}
for a random subset $\mathcal{I}_n(\calU) \subset [N]$ of cardinal $n$.

\subsubsection{Verification of Assumption \ref{ass:Lyap}}
We verify in this section that Assumption \eqref{ass:Lyap} is satisfied for the Lyapunov function $V(\theta)=1+\| \theta \|^2$.
Since $\err(\theta, \calU)$ is globally bounded and $\|\nabla V(\theta)\|^{2} = \|\theta\|^2$ and
\begin{align*}
\|\nabla \log \pi(\theta) \|^2 \lesssim 1 + \|C^{-1} \theta\|^2 \lesssim 1 + \|\theta\|^2 = V(\theta),
\end{align*}
it is straightforward to see that Assumption \eqref{ass:Lyap}.1 and \eqref{ass:Lyap}.2 are satisfied. Finally,
\begin{align*}
\angleBK{\nabla V(\theta), \frac12 \, \nabla \log \pi(\theta)}
&=
-\frac12 \angleBK{\theta, C^{-1} \theta} + \frac12 \, \sum_{i=1}^{N}\text{logit} \big(-y_{i} \angleBK{\theta,x_{i}} \big) \, y_{i} \, \angleBK{\theta, x_{i}}\\
&\leq
-\frac{\lambda_{\min}}{2} \|\theta\|^2 + \frac{\sum_{i=1}^N \|x_i\|}{2} \, \|\theta\|
\leq -\frac{\lambda_{\min}}{4} V(\theta) + \beta
\end{align*}
with $\lambda_{\min} > 0$ the smallest eigenvalue of $C^{-1}$ and $\beta \in (0,\infty)$ the global maximum over $\theta \in \RR^d$ of the function $\theta \mapsto -\frac{\lambda_{\min}}{4} \|\theta\|^2 + \frac{\sum_{i=1}^N \|x_i\|}{2} \, \|\theta\|$.

\subsubsection{Comparison of the SGLD and the MALA for logistic regression}\label{sec:logreg}We consider a simulated dataset where $d=3$ and $N=1000$.  We set the input covariates 
$
x_i=(x_{i,1},x_{i,2},1)
$
with $x_{i,1},x_{i,2}\iid \Normal (0,1)$ for $i=1\dots N$, and use a Gaussian prior $\theta\sim \Normal(0,I)$.  We draw a $\theta_0\sim\Normal(0,I)$ and based on it we generate $y_i$ according to the model probabilities \eqref{eq.logistic}. In the following we compare MALA in SGLD by comparing their estimate for the variance of the first component.

The findings of this article show that SGLD-based expectation estimates converge at a slower rate of at most $n^{-\frac{1}{3}}$ compared to the standard rate of $n^{-\frac{1}{2}}$  for standard MCMC algorithms such as the MALA algorithm. In the following we demonstrate that in the non-asymptotic regime (allowing only a few passes through the data set) the SGLD can be advantageous.  We start both algorithms at the MAP estimator and we ensure that this study is not biased due to different speeds in finding the mode of the posterior. For a fair comparison we tune the MALA to an acceptance rate of approximately $0.564$ following the findings of \cite{Roberts1998OptimalSMala}. For the SGLD-based variance estimate of the first component for $n=30$  we choose $\delta_m=(a\cdot m+b)^{-0.38}$ as step sizes and optimise over the choices of $a$ and $b$.  This is achieved by estimating the MSE for choices of $a$ and $b$ on a log-scale grid based on $512$ independent runs.  The estimates based on $20$ and $1000$ effective iterations through the data set the averages are visualised in the heat maps in Figure \ref{fig:heatmap}. That means we limit the algorithm to $200$ and $1000000$ likelihood evaluations, respectively. The figures indicate that the range of the good parameter choices seems to be the same in both cases. Using the heat map for the estimated MSE after 20 iterations through the data set, we pick  $a=5.89\cdot 10^7$ and $b=7.90\cdot 10^8$ and compare the time behaviour of the SGLD and the MALA algorithm in Figure \ref{fig:LRtrans}. The figure is a simulation evidence that the SGLD algorithm can be advantageous in the initial phase for the first few iterations through the data set. This recommends further investigation as the initial phase can be quite different from the asymptotic phase.
\newcommand{\exedout}{%
  \rule{0.8\textwidth}{0.5\textwidth}%
}
\begin{figure}
\centering
\begin{minipage}{0.50\textwidth}
\centering
\includegraphics[width=1\textwidth]{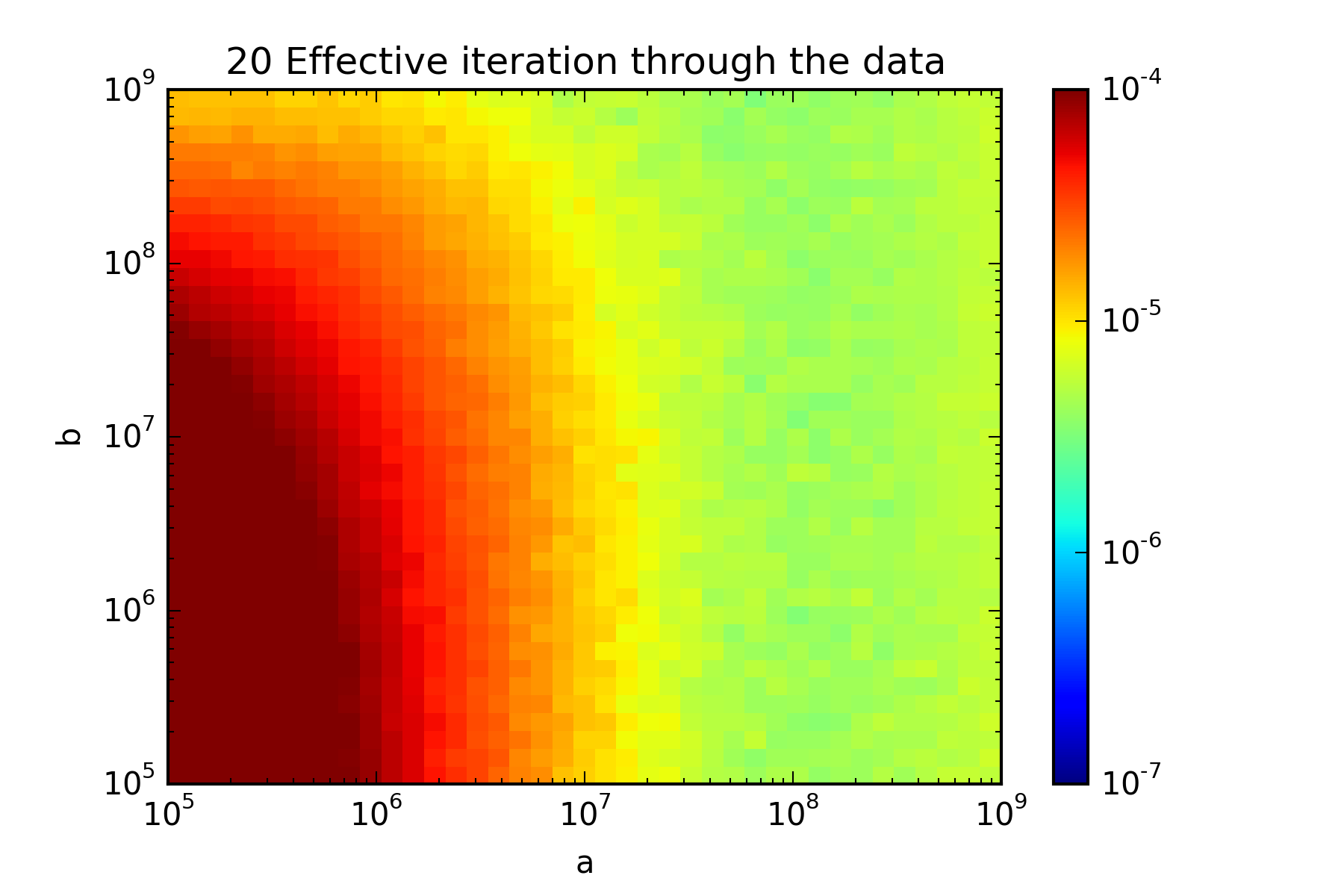}
\end{minipage}\hfill
\begin{minipage}{0.50\textwidth}
\centering
\includegraphics[width=1\textwidth]{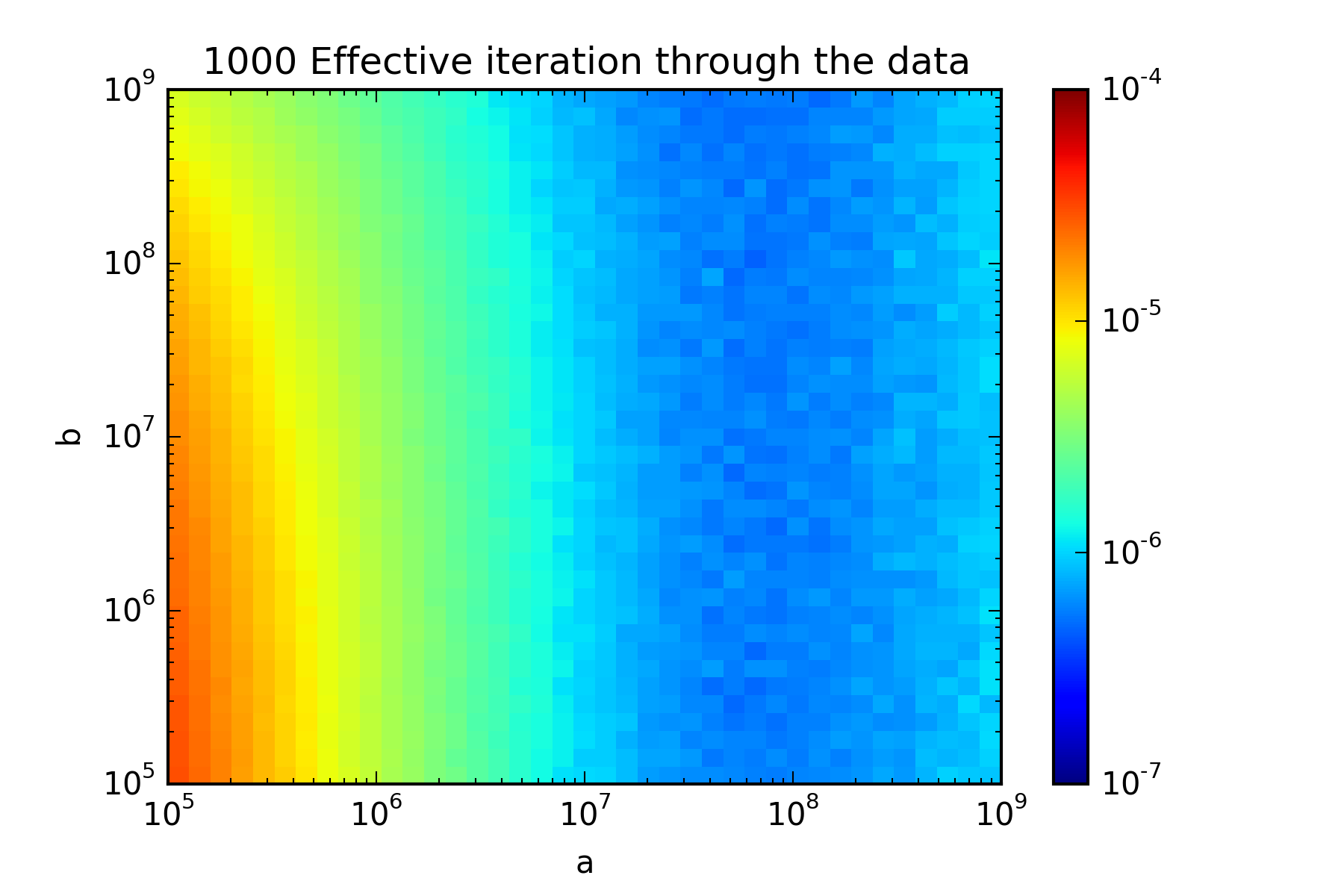}
\end{minipage}
\caption{\label{fig:heatmap}Expected MSE of the SGLD-based estimate variance estimate of the first component for $n=30$ and step sizes $\delta_m=(a\cdot m+b)^{-0.38}$ after 20 and 1000 iterations through the data set}
\end{figure}

\begin{figure}
\centering
\includegraphics[width=0.7\textwidth]{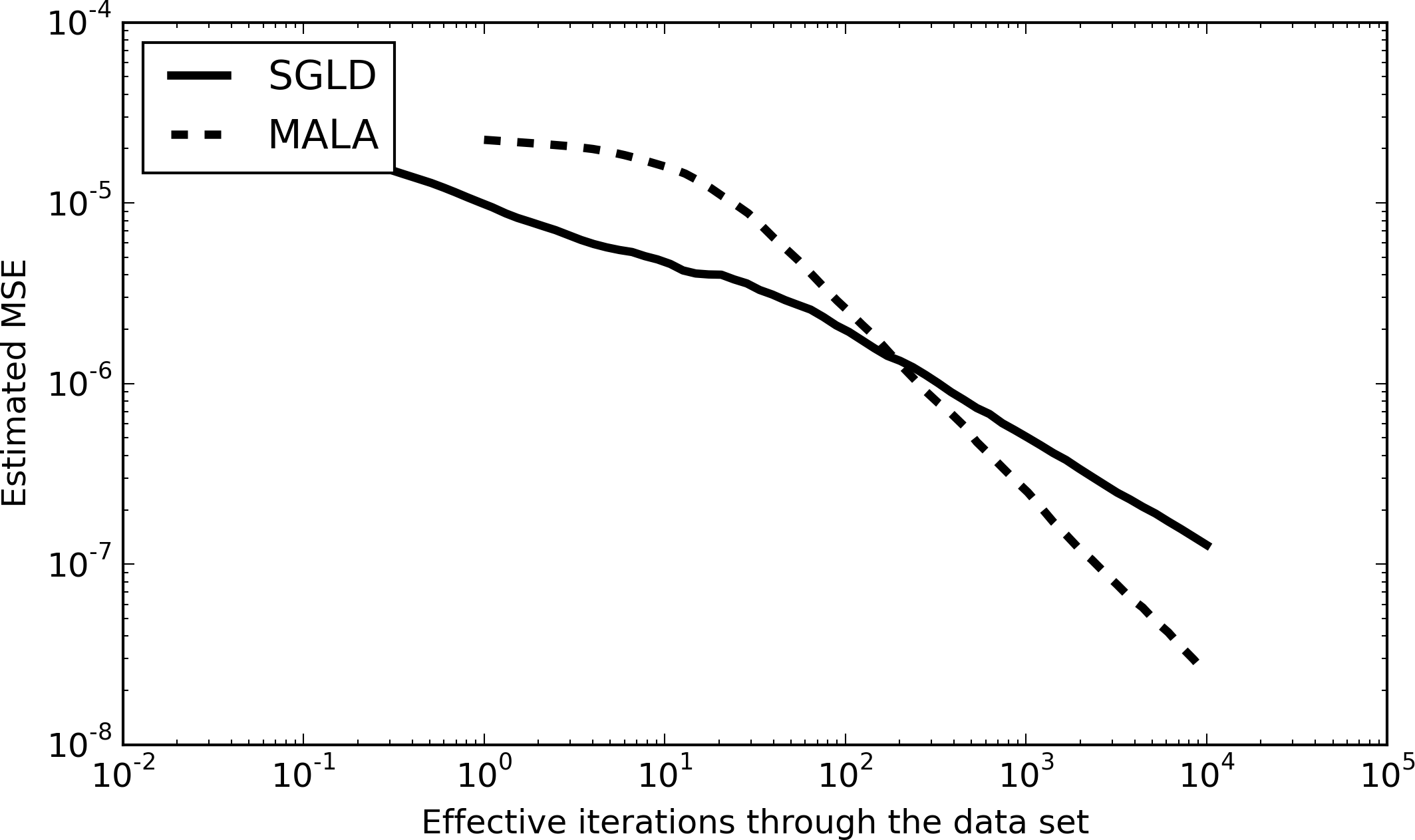}
\caption{\label{fig:LRtrans}Behaviour of the MSE of estimating the posterior variance of the first component for 3-dimensional logistic regression of MALA and SGLD with tuned parameters}
\end{figure}
 
%
%
\section{Conclusion}
So far, the research on the SGLD algorithm has mainly been focused on extending the methodology.  In particular, a parallel version has been introduced in \citet{ahn2014distribuetd} and it has been adapted to natural gradients in \citet{patterson2013stochastic}.
This research has been accompanied by promising simulations. 
In contrast, we have focused in this article on providing rigorous mathematical foundations for the SGLD algorithm by showing that the step-size weighted estimator $\pi_m(f)$ is consistent, satisfies a central limit theorem and its asymptotic bias-variance decomposition can be characterised by an explicit functional $\biasvariance_m$ of the step-sizes sequence $(\delta_m)_{m \geq 0}$. 
The consistency of the algorithm is mainly due to the decreasing step-sizes procedure that asymptotically removes the bias from the discretization and ultimately mitigates the use of an unbiased estimate of the gradient instead of the exact value.
Additionally, we have proved a diffusion limit result that establishes that, when observed on the right (inhomogeneous) time scale, the sample paths of the SGLD can be approximated by a Langevin diffusion.

The CLT and bias-variance decomposition can be leveraged to show that it is optimal to choose a step-sizes sequences $(\delta_m)_{m \geq 0}$ that scales as $\delta_m \asymp m^{-1/3}$; the resulting algorithm converges at rate $m^{-1/3}$.
Note that this recommendation is different from the previously suggested \citet{welling2011bayesian} choice of $\delta_m \asymp m^{-1/2}$.  

Our theory suggests that an optimally tuned SGLD method converges at rate $\mathcal{O}(m^{-1/3})$, and is thus asymptotically less efficient than a standard MCMC procedure.  We believe that this result does not necessarily preclude SGLD to be more efficient in the initial transient phase, a result hinted at in Figure \ref{fig:NonAsympMSE};  
the detailed study of this (non-asymptotic) phenomenon is an interesting venue of research.
The asymptotic convergence rate of SGLD depends crucially on the decreasing step sizes, which is required to reduce the effect of the discretization bias due to the lack of a Metropolis-Hastings correction.  Another avenue of exploration is to determine more precisely the bias resulting from the discretization of the Langevin diffusion, and to study the effect of the choice of step sizes in terms of the trade-off between bias, variance, and computation.

\appendix

%
%
\section{Proof of Lemma \ref{lem.MR}} 
\label{sec.proof.lem.MR}
Recall Kronecker's Lemma \citep[Lemma \textrm{IV}.3.2]{shiryaev1996probability} that states that for a non-decreasing and positive sequence $b_m \to \infty$ and another real valued sequence $(a_m)_{m \geq 0}$ such that the series $\sum_{m \geq 0} a_m / b_m$ converges the following limit holds,
\begin{equation*}
\lim_{m \to \infty} \; \frac{\sum_{k=0}^{m} \, a_{k}}{b_m} \; = \; 0.
\end{equation*}
For proving Equation \eqref{eq.martingale.type} it thus suffices to show that the sums $\sum_{k \geq 0} \abs{ \Delta M_k } / T_k$ and $\sum_{k \geq 0} \abs{ X_k } / T_k$ are almost surely finite. This follows from Condition \eqref{eq.martingale.type.M.part} ($L^2$ martingale convergence theorem) and Condition \eqref{eq.martingale.type.R.part}.

%
%
\section{Proof of Lemma \ref{lem:stability}} 
\label{sec.proof.lem.stability}
For clarity, the proof is only presented in the scalar case $d=1$; the multidimensional setting is entirely similar.
%
%
Before embarking on the proof, let us first mention some consequences of Assumptions \ref{ass:Lyap} that will be repeatedly used in the sequel.
Since the second derivative $V^{''}$ is globally bounded and $(V')^2$ is upper bounded by a multiple of $V$, we have that
\begin{equation} \label{eq.bound.derivative.Vp}
\big| (V^p)^{''}(\theta) \big| \lesssim V^{p-1}(\theta)
\end{equation}
and that the function $V^{1/2}$ is globally Lipschitz. By expressing the quantity $V^p(\theta+\epsilon)$ as $\big(V^{1/2}(\theta) + [V^{1/2}(\theta+\epsilon)-V^{1/2}(\theta)] \big)^{2p}$, it then follows that
\begin{equation} \label{eq.bound.difference.Vp}
V^p(\theta+\epsilon) \lesssim V^p(\theta) + |\epsilon|^{2p}.
\end{equation}
Similarly, Definition \eqref{eq.sgld}, the bound $\| \nabla \log \pp(\theta)  \|^2 \lesssim V(\theta)$ and Equation \eqref{eq.bound.H} yield that for any exponent $0 \leq p \leq p_H$ the following holds,
\begin{equation} \label{eq.bound.delta.theta}
\EE_m[ \, |\theta_{m+1} - \theta_m|^{2p} \,] 
\lesssim 
\delta^{2p}_{m+1} \, V^{p}(\theta) + \delta_{m+1}^{p}.
\end{equation}
For clarity, the proof of Lemma \eqref{lem:stability} is separated into several steps. First, we establish that the process $m \mapsto V^p(\theta_m)$ satisfies a Lyapunov type condition; see Equation \eqref{eq.discr.lyap} below. We then describe how Equation \eqref{eq.stability.estimate} follows from this Lyapunov condition. The fact that $\pi(V^p)$ is finite can be seen as a consequence of Theorem $2.2$ of \citep{roberts1996exponential}.

\begin{itemize}
\item 
{\bf Discrete Lyapunov condition.}\\
\noindent
Let us prove that there exists an index $m_0 \geq 0$ and constants $\alpha_p, \beta_p >0$ such that for any $m \geq m_0$ we have
\begin{equation} \label{eq.discr.lyap}
\EE_m\big[ V^p(\theta_{m+1}) - V^p(\theta_{m}) \big] / \delta_{m+1}
\; \leq \;
- \alpha_p \, V^p(\theta_m) +\beta_p.
\end{equation}
Since for any $\epsilon$ there exists $C_\epsilon$ such that $V^{p-1}(\theta) \leq C_\epsilon + \, \epsilon V^p(\theta)$, for proving \eqref{eq.discr.lyap} it actually suffices to verify that we have
\begin{equation} \label{eq.discr.lyap.weak}
\EE_m\big[ V^p(\theta_{m+1}) - V^p(\theta_{m}) \big] / \delta_{m+1}
\; \leq \;
- \widetilde{\alpha}_p \, V^p(\theta_m) + \widetilde{\beta}_p \, V^{p-1}(\theta_m)
\end{equation}
for some constants $\widetilde{\alpha_p}, \widetilde{\beta_p} > 0$ and index $m \geq 1$ large enough. A second order Taylor expansion yields that the left hand side of \eqref{eq.discr.lyap.weak} is less than
\begin{equation} \label{eq.discr.lyap.2}
\EE_m\big[ (V^p)'(\theta_m) \, (\theta_{m+1}-\theta_{m}) \big] / \delta_{m+1}
+
\frac12 \, 
\EE_m\big[(V^p)^{''}(\xi) \, (\theta_{m+1}-\theta_m)^2 \big] / \delta_{m+1}
\end{equation}
for a random quantity $\xi$ lying between $\theta_m$ and $\theta_{m+1}$. Since $\EE_m[\theta_{m+1}-\theta_m] = \frac12 \, \nabla \log \pp(\theta_m)$, the drift condition \eqref{eq.lyapunov.drift} yields that the first term of \eqref{eq.discr.lyap.2} is less than
\begin{equation} \label{eq.first.term}
p \, V^{p-1}(\theta_m) \, \BK{ -\alpha \, V(\theta_m) + \beta }
\end{equation}
for $\alpha,\beta>0$ given by Equation \eqref{eq.lyapunov.drift}.
Consequently, for proving Equation \eqref{eq.discr.lyap}, it remains to bound the second term of \eqref{eq.discr.lyap.2}. Equation \eqref{eq.bound.derivative.Vp} shows that $|(V^p)^{''}(\xi)|$ is upper bounded by a multiple of $|V^{p-1}(\xi)|$; the bound \eqref{eq.bound.difference.Vp} then yields that $|V^{p-1}(\xi)|$ is less than a constant multiple of $|V^{p-1}(\theta_m)| + |\theta_{m+1}-\theta_m|^{2(p-1)}$. It follows from the bound \eqref{eq.bound.delta.theta} on the difference $(\theta_{m+1} - \theta_m)$ and the assumption $\EE[ \, \|H(\theta, \calU) \|^{2p_H}\,] \lesssim V^{p_H}(\theta)$ that for any $\epsilon > 0$ one can find an index $m_0 \geq 1$ large enough such that for any index $m \geq m_0$ the second term of \eqref{eq.discr.lyap.weak} is less than a constant multiple of 
\begin{equation} \label{eq.second.term}
\epsilon \, V^p(\theta_m) + \beta_{p,\epsilon} \, V^{p-1}(\theta)
\end{equation}
for a constant $\beta_{p, \epsilon} > 0$. Equations \eqref{eq.first.term} and \eqref{eq.second.term} directly yield to Equation \eqref{eq.discr.lyap.weak}, which in turn implies to Equation \eqref{eq.discr.lyap}.

\item
{\bf Proof that $\sup_{m \geq 1} \; \EE[V^p(\theta_m)] \; < \; \infty$ for any $p \leq p_H$.}\\
\noindent
Equations \eqref{eq.bound.difference.Vp} and \eqref{eq.bound.delta.theta} show that if $\EE[V^p(\theta_m)]$ is finite then so is $\EE[V^p(\theta_{m+1})]$. Under the conditions of Lemma \ref{lem:stability}, this shows that $\EE[V^p(\theta_m)]$ is finite for any $m \geq 0$. An inductive argument based on the discrete Lyapunov Equation \eqref{eq.discr.lyap} then yields that for any index $m \geq m_0$ the expectation $\EE[V^p(\theta_m)]$ is less than
\begin{equation}
\max\Big( \beta_p / \alpha_p,\max \big\{ \EE[V^p(\theta_m)]: \; \; 0 \leq m \leq m_0 \big\} \Big). \label{eq:stabilityIndepDelta}
\end{equation}
It follows that $\sup_{m \geq 1} \; \EE[V^p(\theta_m)]$ is finite.
\item
{\bf Proof that $\sup_{m \geq 1} \; \pi_m(V^p)\; < \; \infty$ for any $p \leq p_H/2$.}\\
\noindent
One needs to prove that the sequence $(1/T_m) \sum_{k=m_0}^m \delta_{k+1} V^p(\theta_k)$ is almost surely bounded. The discrete Lyapunov Equation \eqref{eq.discr.lyap} yields that $\delta_{k+1} V^p(\theta_k)$ is less than $\delta_{k+1} \, \beta_p / \alpha_p -\EE_k[ V^p(\theta_{k+1}) - V^p(\theta_k)] / \alpha_p$; this yields that $(1/T_m) \sum_{k=m_0}^m \delta_{k+1} V^p(\theta_k)$ is less than a constant multiple of
\begin{equation*}
1 + \frac{V^p(\theta_{m_0})}{T_m} + \frac{1}{T_m} \sum_{k=m_0}^{m} \Big\{ V^p(\theta_{k+1}) - \EE_{k}[V^p(\theta_{k+1})] \Big\}.
\end{equation*}
To conclude the proof, we prove that the last term in the above displayed Equation almost surely converges to zero; by Lemma \ref{lem.MR}, it suffices to prove that the quantity 
\begin{equation} \label{eq.martingale.a.s.bounded}
\sum_{k \geq m_0}
\E{ \abs{ \frac{V^p(\theta_{k+1}) - \EE_{k}[V^p(\theta_{k+1})]}{T_k} }^2 }
\end{equation}
is almost surely finite. 
We have $\E{ |V^p(\theta_{k+1}) - \EE_{k}[V^p(\theta_{k+1})]|^2 } \leq 2 \times \EE[| \, V^p(\theta_{k+1})-V^p(\theta_{k}) \, |^2]$ and the mean value theorem yields that $| V^p(\theta_{k+1})-V^p(\theta_{k}) | \lesssim V^{p-1}(\xi) \, V'(\xi) \, (\theta_{k+1}-\theta_{k})$ for some $\xi$ lying between $\theta_k$ and $\theta_{k+1}$. The bound $|V'(\theta)| \lesssim V^{1/2}(\theta)$ and Equation \eqref{eq.bound.difference.Vp} then yield that 
$| V^p(\theta_{k+1})-V^p(\theta_{k}) | \lesssim V^{p-1/2}(\theta_k) \, \big|\theta_{k+1}-\theta_k\big| +  \big|\theta_{k+1}-\theta_k\big|^{2p}$. From the bound \eqref{eq.bound.delta.theta} and the assumption that $\EE[H(\theta,\calU)^{2p_H}] \lesssim V^{p_H}(\theta)$ it follows that the quantity in Equation \eqref{eq.martingale.a.s.bounded} is less than a constant multiple of
\begin{equation*}
\sum_{k \geq m_0}
 \frac{\EE\big[ \, V^{2p}(\theta_k) \, \big] \times \delta_k}{T^2(k)}.
\end{equation*}
Since $\E{  V^{2p}(\theta_k) }$ is uniformly bounded for any $p \leq p_H / 2$ and $\sum_{m \geq m_0} \delta_m / T^2(m) < \infty$ (because the sum $\sum_m T^{-1}(m+1) - T^{-1}(m)$ is finite), the conclusion follows.

\item 
{\bf Proof of $\pi(V^p)\; < \; \infty$ for any $p \geq 0$.}\\
\noindent
Since $V(\theta) \lesssim 1 + \|\theta\|^2$, the drift condition \eqref{eq.lyapunov.drift} yields that Theorem $2.1$ of \citep{roberts1996exponential} holds. Moreover, the bound $V^{p-1}(\theta) \leq C_{\epsilon} + \epsilon \, V^p(\theta)$ implies that there are constants $\alpha_{p,*}. \beta_{p,*} >0$ such that
\begin{equation}
\A V^p(\theta) \leq -\alpha_{p,*} \, V^p(\theta) + \beta_{p,*}
\end{equation}
where $\A$ is the generator of the Langevin diffusion \eqref{eq:overdampedLangevin}. Theorem $2.2$ of \citep{roberts1996exponential} gives the conclusion.
\end{itemize}

\noindent
{\bf Proof that $\sup_{m \geq 1} \; \pi^{\omega}_m(V^p)\; < \; \infty$ for any $p \leq p_H/2$.}\\
\noindent
One needs to prove that the sequence 
$[1/\Omega_m] \times \sum_{k=m_0}^m \omega_{k+1} V^p(\theta_k)$  
is almost surely bounded. The bound $\delta_{k+1} V^p(\theta_k) \lesssim \delta_{k+1} \, \beta_p / \alpha_p -\EE_k[ V^p(\theta_{k+1}) - V^p(\theta_k)] / \alpha_p$ yields that $\pi^{\omega}_m(V^p)$ is less than a constant multiple of
\begin{eqnarray*}
1 + \frac{(\omega_{m_0} / \delta_{m_0}) \, V^p(\theta_{m_0})}{T_m} 
&+ \Omega^{-1}(m) \, \sum_{k=m_0+1}^{m} (\omega_k / \delta_k) \, \Big\{ V^p(\theta_{k+1}) - \EE_{k}[V^p(\theta_{k+1})] \Big\} \\
&+ \Omega^{-1}(m) \,  \sum_{k=m_0}^{m-1} \Delta(\omega_k / \delta_k) \, V^p(\theta_{k}).
\end{eqnarray*}
To conclude the proof, we establish that the following limits hold almost surely,
\begin{eqnarray}
\label{eq.martingale.cv.0.weighted}
\,& \lim_{m \to \infty}
\, \Omega^{-1}(m) \, \sum_{k=m_0+1}^{m} (\omega_k / \delta_k) \, \Big\{ V^p(\theta_{k+1}) - \EE_{k}[V^p(\theta_{k+1})] \Big\} = 0\\
\label{eq.limsup.weighted}
\,& \lim_{m \to \infty} \;
\, \Omega^{-1}(m) \,  \sum_{k=m_0}^{m-1} \Delta(\omega_k / \delta_k) \, V^p(\theta_{k}) = 0.
\end{eqnarray}
To prove Equation \eqref{eq.martingale.cv.0.weighted} it suffices to use the assumption that $\sum_{m \geq 0} \omega^2_m / [\delta_m \Omega^2_m] < \infty$ and then follow the same approach used to establish that the quantity \eqref{eq.martingale.a.s.bounded} is finite. Lemma \ref{lem.MR} shows that to prove Equation \eqref{eq.limsup.weighted} it suffices to verify that
\begin{eqnarray*}
\EE \Big[ \sum_{m \geq 0} \, \big| \Delta(\omega_m / \delta_m) \big| \, V^p(\theta_{m}) / \Omega_m \Big] < \infty.
\end{eqnarray*}
This directly follows from the assumption that $\sum_{m \geq 0} \, \big| \Delta(\omega_m / \delta_m) \big| / \Omega_m < \infty$ and the fact that $\sup_{m \geq 0} \, \EE[V^p(\theta_{m})]$ is finite.

\vskip 0.2in
\bibliography{stochasticGradients}

\end{document}